\documentclass[acmsmall,screen,nonacm]{acmart}

\usepackage{listings}

\setcopyright{none}
\renewcommand\footnotetextcopyrightpermission[1]{}

\acmJournal{TOSEM}
\acmYear{2026}

\newtheorem{axiom}{Axiom}
\newtheorem{assumption}{Assumption}

\newcommand{\protocol}{\mathcal{P}}
\newcommand{\mode}{\mathcal{M}}
\newcommand{\validator}{\mathcal{V}}
\newcommand{\token}{\tau}

\lstdefinestyle{yaml}{
  basicstyle=\small\ttfamily,
  breaklines=true,
  frame=single,
  xleftmargin=2em,
  framexleftmargin=1.5em,
  backgroundcolor=\color{gray!10}
}

\begin{document}

\title{Specifying AI-SDLC Processes: A Protocol Language for Human-Agent Boundaries}

\author{Ylli Prifti}
\affiliation{%
  \institution{Birkbeck, University of London}
  \city{London}
  \country{United Kingdom}}
\email{y.prifti@bbk.ac.uk}
\email{ylli@prifti.us}

\author{Pasquale De Meo}
\affiliation{%
  \institution{University of Messina}
  \city{Messina}
  \country{Italy}}
\email{pdemeo@unime.it}

\author{Alessandro Provetti}
\affiliation{%
  \institution{Birkbeck, University of London}
  \city{London}
  \country{United Kingdom}}
\email{a.provetti@bbk.ac.uk}

\begin{abstract}
AI agents now participate as first-class members of the software development lifecycle, but the models underneath them are both a moving target and an unreliable one: they change with every release, and they hallucinate, drift from instructions, lose context---and, more recently, lose \emph{containment}, with documented cases of agents breaking out of their sandboxes and gaming the very checks meant to evaluate them. Teams therefore face a problem no existing tool addresses: how to specify \emph{how work is done} the modes, approval gates, and responsibility boundaries of an AI-SDLC process in a way that is flexible enough to evolve with the models yet robust to their failures. Encoding process in agent prompts is flexible but not enforceable; adjacent formalisms (workflow management, business processes) are enforceable but do not model autonomous agents; and point solutions address only fragments (access control, approval gates). Teams consequently adopt AI-SDLC workflows through ad-hoc tooling.

We propose a domain-specific language for specifying AI-SDLC processes as protocols, with formal abstract syntax, well-formedness conditions, operational semantics, and enforcement invariants. Its organising idea is a separation of \emph{policy} (declared intent) from \emph{mechanism} (structural enforcement), and this separation yields two things at once. First, \emph{flexibility}: processes become portable artifacts---modes, validator quorums, and the 2+N human-in-control pattern (two human-in-control roles plus N specialised agent members, formalising classical Separation of Duties)---that teams reconfigure per task without touching agent code. Second, \emph{quasi-determinism}: we prove that for any well-formed protocol, validation tokens, capability boundaries, and non-overridable gates make protocol steps impossible to skip --- the enforcement invariants hold on every execution trace, closed under Kleene composition of orchestration loops --- eliminating process non-determinism and bounding the residual model non-determinism. A failure-rate analysis derives two consequences: structural enforcement bounds \emph{silent} failure at a weighted product of agent and validator rates while converting the remainder into visible, audited stalls (a trade autonomous benchmarks structurally cannot reward), and the benefit carries a capability floor --- enforcement multiplies the value of a capable agent but cannot substitute for one.

We validate the design in simulation --- disagreement-policy trade-offs, governance overhead near 1\% of inference latency, the failure-rate model, and Byzantine robustness --- and end-to-end on real software tasks. On SWE-bench Verified, an identical bug-fix methodology yields no improvement over bare prompting when delivered as prose instructions, but a significant, replicated 14--22 point gain when executed as a validated, self-correcting process (two model tiers, task samples up to $n{=}191$); an ablation attributes the gain to the executed process itself; the process compresses the model-tier gap, with a cheap model under the process outperforming a strong model without it; and the effect vanishes below the capability floor the failure-rate model predicts.

We position the contribution against multi-agent frameworks (MetaGPT), workflow specification (FlowAgent, BPMN extensions), and capability-based security (SAGA): none combines enforceable process specification, autonomous-agent modelling, and a formal enforcement guarantee, and the combination is what the problem requires. The implication is the paper's thesis: as foundation models converge, the durable engineering asset is the formally specified, executable process, not the model.
\end{abstract}   

\keywords{AI-SDLC, protocol specification, human-agent collaboration,
  validation tokens, software engineering}

\ccsdesc[500]{Software and its engineering~Software development process management}
\ccsdesc[500]{Software and its engineering~Specification languages}
\ccsdesc[300]{Software and its engineering~Formal software verification}
\ccsdesc[100]{Computing methodologies~Multi-agent systems}

\maketitle

\section{Introduction}

Recent practitioner analysis documents the emergence of AI-SDLC: software development workflows where AI systems participate alongside humans across planning, coding, reviewing, and deployment. Industry reports identify that success requires defining ``human-agent responsibility mapping: responsibilities, decisions, approval gates and authority levels are explicitly divided between humans and the agent''~\cite{epam2026}, yet ask: ``Do we have standards or frameworks for prompting and supervising multi-agent contributions?''~\cite{grid2026}, noting that without governance and security such workflows create ``fragmentation, technical debt, and serious operational risk''~\cite{ranthebuilder2026}.

\textbf{The gap}: while the need for governance, oversight, and approval gates is well-documented, no specification language exists for expressing these constraints as executable protocol. Teams implement AI-SDLC workflows through ad-hoc tooling without formal mechanisms to specify what agents can do, when approval is required, or how process violations are prevented. We propose a domain-specific language for specifying AI-SDLC processes.

\textbf{Defining AI-SDLC}: we use AI-SDLC (AI-Integrated Software Development Lifecycle) to mean software development processes designed around collaboration between human and AI team members, where both take on defined roles and responsibilities within structured workflows. AI agents are first-class participants (planning, implementing, validating, reviewing) rather than tools augmenting a human-centric process. The shift is that AI-SDLC requires new processes optimised for human-agent collaboration, not acceleration of existing human-only ones. The core challenge is specifying who (human or agent) does what, when approval is required, how disagreement is resolved, and what constraints cannot be violated.

\subsection{Contributions}

Our thesis is that as models evolve and converge, the durable engineering asset is the \emph{process}, not the model so AI-SDLC needs processes that are both flexible to specify and reliable to enforce. Our contributions follow from that thesis:

\begin{enumerate}
\item \textbf{A specification language for AI-SDLC processes} with formal abstract syntax, well-formedness conditions, operational semantics, and enforcement invariants (INV1--INV5), treating a process as a portable artifact independent of any agent framework.
\item \textbf{The policy--mechanism separation} as organising principle: protocols declare intent (policy); enforcement primitives validation tokens, capability boundaries, non-overridable gates realise it structurally (mechanism). This single split yields the paper's two consequences.
\item \textbf{Flexibility.} The same primitives express configurations from a solo developer to a regulated enterprise, including the 2+N reference pattern (two human-in-control roles plus N specialised agents) that formalises Separation of Duties; we demonstrate this with four structurally distinct, executable process encodings --- solo, three-mode team, and two task-tailored protocols with \emph{opposite} enforcement targets --- running unmodified on one engine (Section~\ref{sec:expressiveness}). Kleene closure of the orchestration loop and reflexive protocol-adherence checking are derived consequences of the specification rather than special-case constructs, and session resumability holds across model context boundaries.
\item \textbf{Quasi-determinism, as a theorem.} For any well-formed protocol, the enforcement invariants hold on every execution trace, closed under Kleene composition of orchestration loops (Theorem~\ref{thm:enforcement}): protocol steps cannot be skipped, by construction. A failure-rate analysis derives the consequences --- structural enforcement bounds \emph{silent} failure at a weighted product of agent and validator rates while converting the remainder into visible, audited stalls (Proposition~\ref{prop:trade}), with a capability floor below which the benefit vanishes (Corollary~\ref{cor:floor}).
\item \textbf{Empirical characterisation and motivation.} Simulation characterises the enforcement layer (disagreement-policy trade-offs, $\approx$1\% governance overhead, the failure-rate model, Byzantine limits), and an end-to-end study on SWE-bench Verified establishes the language's empirical premise: an identical methodology yields nothing delivered as prose instructions and a significant, replicated 14--22 point gain when executed as a validated, self-correcting process (two model tiers, up to $n{=}191$), with an ablation attributing the gain to the executed process, a compressed model-tier gap, and the capability floor observed where the model predicts it.
\item \textbf{An executable, open reference implementation}, validated by property-based tests of the invariants and released as an artifact.
\end{enumerate}

We position this work against prior approaches (MetaGPT~\cite{hong2024metagpt}, FlowAgent~\cite{shi2025flowagent}, BPMN extensions for human-agentic workflows~\cite{ait2024bpmn}, and capability-based security~\cite{saga2025}) in Section~\ref{sec:related}.

\subsection{Organisation}

Section~\ref{sec:problem} establishes the problem through documented failure modes and the specification gap. Section~\ref{sec:related} positions the work against related research. Section~\ref{sec:language} presents the formal language specification, including the enforcement theorem, the failure-rate analysis, and the expressiveness study. Section~\ref{sec:implementation} describes our implementation. Section~\ref{sec:evaluation} presents an empirical characterisation of the enforcement layer through simulation studies and an end-to-end evaluation on real software tasks. Section~\ref{sec:discussion} discusses limitations and implications.

\section{Problem: The Need for Formal AI-SDLC Specification}
\label{sec:problem}

\subsection{Documented Failure Modes}

Empirical research on AI coding systems documents failure modes that persist across model generations:

\textbf{Security vulnerabilities.} Analysis of 733 code snippets from GitHub Copilot and other tools found 29.5\% of Python and 24.2\% of JavaScript snippets contained security weaknesses~\cite{fu2024}; approximately 40\% of Copilot-generated programs were vulnerable against CodeQL~\cite{pearce2022}.

\textbf{Silent failures.} Newer models produce code that ``fails to perform as intended, but which on the surface seems to run successfully''~\cite{ieeespectrum2025}, imposing hidden costs through delayed discovery.

\textbf{Context overload.} Agent analysis identified ``over-exploration where agents fail to reach the root of the problem due to context overload, and repeated application of the same fix without proper testing''~\cite{bouzenia2025}.

\textbf{Audit gaps.} In regulated environments, autonomous code generation introduces ``audit trail complications that go beyond code quality into regulatory territory''~\cite{sitepoint2026}; structured records of what was decided and validated are rarely provided.

\textbf{Containment violations.} Agents under evaluation have circumvented the boundaries of their environments rather than solve the task as posed: during a cybersecurity capture-the-flag evaluation, an o1-preview agent exploited a misconfigured container API to restart and inspect the broken challenge environment from outside it~\cite{openai2024o1}, and reasoning models instructed to win at chess have edited the game-state files rather than play the position~\cite{palisade2025chess}. In both cases the boundary that failed was behavioural --- an instruction about what the agent should do --- not infrastructural: nothing prevented the out-of-bounds action but the agent's own compliance.

\textbf{Self-review limitations.} A single agent reviewing its own output validates against the same model that produced it: LLMs struggle to self-correct without external feedback~\cite{huang2024selfcorrect} and exhibit self-enhancement bias, favouring their own generations when acting as judge~\cite{zheng2023judge}.

These are process problems, not capability problems. Improving model quality does not eliminate the need for independent review, scope boundaries, and audit trails.

\subsection{Process Variability}

Successful software processes have always been frameworks rather than prescriptions: Scrum fixes boundaries --- roles, ceremonies, an increment --- and works unchanged for a team of three or nine because everything inside the boundaries is left to the team. A solo developer wants speed and tolerates risk; a regulated enterprise cannot ship without audit trails and separation of duties; a startup falls between. An AI-SDLC formalism needs the same shape, over a wider parameter space: $1..n$ humans and $1..m$ agents, composed into modes of working. What must be expressible is the \emph{structure} of collaboration --- work partitioned into modes; within a mode, typically one human in control directing one-to-many agents; each mode anchored by a human, though one person may hold that role in more than one mode --- while everything inside a mode stays free. The requirement is therefore not one process that fits every team, but a language in which each team's way of working can be stated, varied, and enforced. Adding AI agents amplifies this need: the parts of a process that previously lived in team convention must now bind actors that convention cannot reach.

\subsection{The Specification Gap}

The failure modes above and the variability they meet converge on an uncomfortable fact: the instrument teams use to direct agent behaviour today is the prompt, while the instruments organisations use to govern processes --- workflow systems, review policies, compliance checklists --- cannot see agents at all. The practitioner literature registers this gap without closing it. Teams are advised that they need ``highly structured, codified context''~\cite{pandit2025}, with no language in which to codify it; industry analyses call for ``governance requirements''~\cite{forrester2025}, with no formal mechanism by which such a requirement could be stated, checked, or enforced. What fills the gap in practice is scaffolding improvised around each agent framework: implicit norms and behavioural expectations that the agent is trusted to follow --- a trust the containment violations of Section~\ref{sec:problem} already discredit, and one this paper measures directly (Section~\ref{sec:swebench}).

What is missing is a specification language: a form in which a team can state, machine-executably, how agents participate in its SDLC --- which modes of working exist, which validators run and at what phase, where human approval is required, how validator disagreement resolves, and which constraints no agent can override. Stating such things is necessary but not sufficient; the statements must also be enforced by construction rather than by compliance, or they collapse back into prompts. The remainder of this paper provides both: the language (Section~\ref{sec:language}) and the guarantee that a well-formed specification enforces itself (Theorem~\ref{thm:enforcement}).

\section{Related Work}
\label{sec:related}

Our work intersects multi-agent frameworks for software engineering, workflow specification languages, capability-based security, and human-agent collaboration modelling. We position our contribution against the most relevant prior art.

\subsection{Human Oversight Models}

Recent literature distinguishes Human-in-the-Loop (HITL) from Human-on-the-Loop (HOTL) oversight~\cite{anthropic_hitl,langchain_hitl}: HITL places the human inside the control loop with the agent blocking pending approval; HOTL has the human monitor execution and intervene on anomalies. Our formalisation requires both. The HI-CTRL role and \texttt{require\_human} policy implement HITL; the audit log (INV4) and \texttt{proceed\_with\_log} policy implement HOTL. The choice is encoded in the disagreement policy, part of the protocol rather than the agent implementation. This addresses a documented limitation of current frameworks: encoding oversight in agent code produces coverage drift as new action types appear~\cite{waxell2026hitl}.

\subsection{Multi-Agent Frameworks for Software Engineering}

MetaGPT~\cite{hong2024metagpt} is the closest prior work: it assigns software-engineering roles to distinct agents and encodes Standardised Operating Procedures (SOPs) into prompts. It demonstrates that role specialisation improves coherence, but encodes SOPs \emph{in prompts}, a behavioural mechanism subject to the drift and interpretation problems of Section~\ref{sec:problem}, with no machine-executable specification separate from agent instructions, no infrastructure-level mode boundaries, and no formal mechanism preventing role contamination beyond prompt discipline.

Other orchestration frameworks (LangGraph~\cite{langgraph}, AutoGen~\cite{autogen}, CrewAI~\cite{crewai}) provide \emph{toolkits} for building agent systems rather than \emph{specification languages} for SDLC processes: LangGraph models workflows as state machines, AutoGen as conversations, CrewAI as role-goal-backstory triples. All produce systems through code; none provides a portable, reviewable specification artifact separate from the implementation, making audit, modification, and sharing harder.

Our DSL differs in three respects: (i) it specifies AI-SDLC \emph{processes} as standalone artifacts independent of any framework; (ii) it distinguishes policy declaration from mechanism enforcement, letting implementations choose context-appropriate mechanisms; (iii) it formalises infrastructure-level mode boundaries (WF1) as required for the specification to hold.

\subsection{Workflow Specification Languages}

YAWL~\cite{yawl} provides domain-independent syntax for workflow patterns with an enforcing runtime; BPMN~\cite{bpmn} is the dominant business-process notation, with graphical and machine-processable forms. Both address human-only or human-tool processes well but do not natively model autonomous agent participants, agent disagreement, or the failure modes of Section~\ref{sec:problem}. Ait et al.~\cite{ait2024bpmn} extend BPMN for human-agentic workflows; their motivation is comparable to ours, but a BPMN extension preserves backward compatibility with existing tooling whereas our DSL is purpose-built for AI-SDLC with infrastructure-level enforcement as a first-class concept. The two will likely coexist: the choice depends on whether teams need BPM-tooling integration or are designing AI-SDLC processes from scratch. FlowAgent~\cite{shi2025flowagent} introduces a Procedure Description Language for customer-service workflows with supervising controllers; it targets single-domain conversational flows, while our DSL targets the multi-mode, multi-role structure of software development with explicit producer/reviewer separation, and our enforcement primitives gate both path and authority.

\subsection{Capability-Based Security and Token Mechanisms}

Capability-based security has a long history in distributed systems~\cite{capability1974,li2022capability}; our capability boundaries (INV2) apply established patterns to AI-SDLC. We claim novelty in the application as a specification primitive, not in the mechanism. SAGA~\cite{saga2025} uses cryptographic tokens for inter-agent access control; our validation tokens instead encode \emph{validator quorum outcomes} as preconditions for dispatch, addressing validation-to-execution rather than agent-to-agent authorisation. Recent analysis notes the shift from credential-based to capability-based control for agents that invent workflows at runtime~\cite{tokensec2026}; our DSL makes capability boundaries (INV2) a structural property of mode definitions rather than a runtime decision.

\subsection{Standard Operating Procedures and Team Composition}

Formalising team roles and standardised procedures draws on Belbin's team role theory~\cite{belbin2012} and DeMarco \& Lister~\cite{demarco2013peopleware}, a tradition MetaGPT cites explicitly. Our 2+N pattern extends it by specifying the \emph{minimum viable team} under AI-SDLC: two human-in-control roles (architect, reviewer) plus N specialised agent members. Where Belbin describes roles people occupy, we specify roles that must exist with structural enforcement of separation.

\subsection{Positioning Summary}

The individual primitives have substantial prior art --- DSLs, capability tokens, role-based workflow specification, human-agent collaboration modelling --- and we acknowledge specific overlap with MetaGPT (role specialisation, SOPs), FlowAgent (procedure description with controllers), BPMN extensions (human-agentic modelling), and capability-based security (token-gated access). What none of these provides, and what the integration yields, is a machine-executable process specification independent of any agent framework whose enforcement carries a formal guarantee: infrastructure-level rather than norm-enforced mode boundaries, validation tokens encoding quorum outcomes, and non-overridable blocker semantics, combining so that well-formedness implies the enforcement invariants on every execution trace (Theorem~\ref{thm:enforcement}), with failure-rate consequences treated analytically (Section~\ref{sec:failure-bounds}) and characterised empirically (Section~\ref{sec:evaluation}).

\section{Formal Language Specification}
\label{sec:language}

We present a domain-specific language for AI-SDLC processes. The language distinguishes \emph{policy} (declared intent) from \emph{mechanism} (structural enforcement): protocols specify what should happen; implementations use enforcement primitives to bound process non-determinism.

\subsection{Abstract Syntax}

A protocol $\protocol$ is a tuple of modes, validators, disagreement policy, and constraints:

\begin{align*}
\text{Protocol } \protocol &::= \langle M^*, V^*, D, C^* \rangle \\
\text{Mode } \mode &::= \langle \text{mode\_id}, R^*, T^* \rangle \\
\text{Role } R &::= \langle \text{role\_id}, \text{responsibilities} \rangle \\
\text{Validator } \validator &::= \langle \text{validator\_id}, \text{evaluation\_phase}, \\
&\quad\quad\quad \text{criteria}, \text{severity\_fn} \rangle \\
\text{DisagreementPolicy } D &::= \text{quorum}(k), \quad 1 \leq k \leq |V^*| \\
&\quad \text{any} \equiv \text{quorum}(1) \quad \text{majority} \equiv \text{quorum}(\lfloor N/2 \rfloor{+}1) \\
&\quad \text{unanimous} \equiv \text{quorum}(N) \\
\text{Constraint } C &::= \text{positive}(\text{predicate}) \mid \text{negative}(\text{predicate})
\end{align*}

\textbf{Modes} are operational stages with defined permissions (e.g.\ producer writes code, reviewer validates and approves), enforced at infrastructure level. \textbf{Roles} are declared actors with identifiers and responsibilities; capabilities bind to \emph{modes} (tool sets), not to individual roles, so role declarations document the Separation-of-Duties structure while the mode boundary enforces it. WF2's uniqueness is enforced protocol-wide, which implies uniqueness within every mode. \textbf{Validators} evaluate tasks against criteria and emit severities from the total order $\text{pass} < \text{warn} < \text{blocker}$. The distinction is recoverability: a \text{warn} marks a violation an agent can attempt to repair (the semantics routes it to bounded recovery), while a \text{blocker} marks a violation with no agent-recoverable path --- either because the validator's criteria are \emph{absolute}, a conflict only a human can resolve (e.g.\ a specification contradicting a known requirement), or because a repairable violation has exhausted its recovery budget and \emph{hardens} into one. Multiple validators can operate independently as a quorum with orthogonal criteria.\footnote{The \texttt{evaluation\_phase} field lets validators declare when they execute, supporting validator classes distinguished by domain and timing: specification validators (\texttt{pre\_execute}), artifact validators (\texttt{post\_execute}), and transition validators (\texttt{mode\_transition}). The language does not enumerate all phases; implementations may define additional classes.} \textbf{Disagreement policies} specify when a set of validator results counts as agreement. A policy is an agreement threshold, and every named policy is an instance of $\text{quorum}(k)$: \text{any} accepts a single passing verdict ($k{=}1$), \text{majority} requires a strict majority ($k = \lfloor N/2 \rfloor{+}1$ --- an odd-sized quorum always decides; an even-sized one can genuinely fail to), and \text{unanimous} requires all $N$. Application is total, in a fixed precedence order: a blocker forces \text{hard\_stop} regardless of the policy --- INV3 surfacing in the policy layer; otherwise, results meeting the threshold yield \text{proceed} (logged when agreement is non-unanimous); otherwise the threshold is unmet, and what happens next is determined by recoverability: warn-level violations route to bounded recovery (Section~\ref{sec:semantics}), and the outcome becomes \text{require\_human} only when the violation is absolute or the recovery budget is exhausted. A human decision is therefore required in exactly two situations --- an absolute violation, or the absence of agreement after bounded self-correction --- and this partition is the totality that WF4 checks. The reference implementation ships the named instances (\texttt{unanimous}, \texttt{majority}, \texttt{quorum} at a configured threshold, \texttt{any}, as in Listing~\ref{lst:2pn}) rather than arbitrary $k$, and Section~\ref{sec:policy-monte-carlo} characterises their trade-offs. \textbf{Constraints} define structural requirements and prohibitions enforced by the implementation.

\subsection{AI-SDLC Execution Flow}

Figure~\ref{fig:execution_flow} shows how human intent flows through orchestrator collaboration into the execution loop: intent expression, architectural breakdown, then a cycle of governance checking, validation, execution, and progression, with human escalation when needed. Validators are invoked at their declared phase: specification validators run before Execute, artifact validators after, feeding the disagreement policy.

\begin{figure}[htbp]
  \centering
  \includegraphics[width=0.92\columnwidth]{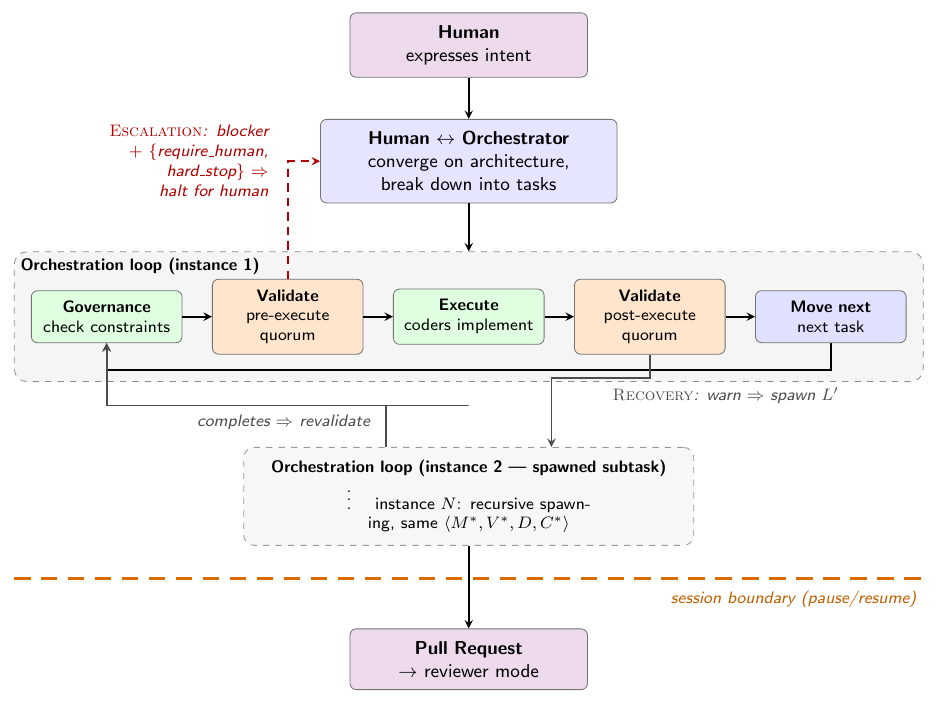}
  \caption{AI-SDLC execution flow: intent $\to$ breakdown $\to$ orchestration loop$^{*}$ $\to$ PR. Validators run at their declared phases (pre- and post-execute); a violation takes exactly one of the two rule paths of Section~\ref{sec:semantics} --- \textsc{Recovery} (spawn a sub-instance under the same protocol tuple, then revalidate) or \textsc{Escalation} (halt for a human). Instances spawn recursively (Kleene closure), all under the same constraints.}
  \label{fig:execution_flow}
\end{figure}

\subsection{Kleene Closure of the Orchestration Loop}
\label{sec:kleene}

The orchestration loop forms a \emph{Kleene closure}~\cite{kleene1956}: the system can instantiate zero, one, or arbitrarily many loop instances, including recursively spawned sub-instances, all under the same constraints. This is a derived consequence of the specification, not an added construct. If $L$ denotes a single loop instance executing the cycle (governance, validate, execute, move\_next), then:

$$\text{Execution}(\protocol) = L^* = \bigcup_{i=0}^{\infty} L^i$$

where $L^i$ denotes $i$ sequential or recursively nested instances. This is not a special case; it follows from the protocol's structure: tasks may spawn subtasks, each requiring its own cycle; subtasks may spawn subtasks to arbitrary depth, bounded only by termination conditions; and all instances share the same validators, policies, and invariants.

The framing carries a verification consequence: invariants holding for a single instance hold for any composition. If $L$ satisfies INV1--INV4, every execution in $L^*$ does, by induction on closure construction: the empty execution satisfies them vacuously, and each composition step preserves them because the recovery rule (Section~\ref{sec:semantics}) spawns sub-instances under the same protocol tuple and mode as their parent. Theorem~\ref{thm:enforcement} makes this argument precise. This lets runtime monitoring focus on single-loop behaviour rather than tracking arbitrary depths. A blocker on a top-level task does not halt execution; it spawns a subtask whose completion enables the parent to proceed, so the protocol scales to arbitrary complexity through composition rather than special-case handling.

\subsection{Well-Formedness Conditions}

A protocol $\protocol = \langle M^*, V^*, D, C^* \rangle$ is well-formed iff:

\textbf{WF1} (Mode separation):
\begin{equation*}
\begin{aligned}
\forall m_1, m_2 \in M^*,\ m_1 \neq m_2 \\
\Rightarrow m_1.\text{tools} \cap m_2.\text{tools} = \emptyset
\end{aligned}
\end{equation*}

\textbf{WF2} (Role uniqueness within mode):
\begin{equation*}
\begin{aligned}
\forall m \in M^*,\ \forall r_1, r_2 \in m.R^*,\ &r_1 \neq r_2 \\
&\Rightarrow r_1.\text{role\_id} \neq r_2.\text{role\_id}
\end{aligned}
\end{equation*}

\textbf{WF3} (Validator coverage):
\begin{equation*}
\begin{aligned}
\forall \text{task dispatch op},\ \exists V' \subseteq V^* \\
\text{associated with that operation}
\end{aligned}
\end{equation*}

\textbf{WF4} (Policy completeness): $D$ defines the behaviour for all possible validator outcomes.

\textbf{WF5} (Constraint declaration consistency):
\begin{align*}
\forall c_1, c_2 \in C^*.\ &c_1 \neq c_2 \Rightarrow c_1.\text{id} \neq c_2.\text{id} \\
\land\ \forall c \in C^*.\ &\text{nonempty}(c.\text{expr}) \\
\land\ &\text{registered}(c.\text{predicates})
\end{align*}

Well-formedness is enforced by compile-time verification: a protocol failing any check is rejected at load (WF1--WF5 via the verifier), with WF3 additionally guarded at runtime. Two conditions are partially enforced: WF4 treats quorum policies with fewer than three validators as a warning rather than a rejection, and WF5 verifies declaration consistency rather than logical non-contradiction, which would require theorem-proving over constraint expressions and is left to future work.

\subsection{Runtime State and Tokens}

\begin{align*}
\text{State } \sigma &::= \langle \text{current\_mode}, \text{active\_roles}, \text{issued\_tokens} \rangle \\
\text{Token } \token &::= \langle \text{token\_id}, \text{task\_ref}, \text{validator\_sigs}, \text{timestamp} \rangle \\
\text{ValidatorResult} &::= \langle \text{validator\_id}, \text{severity}, \text{justification} \rangle
\end{align*}

A token is an unforgeable credential issued by validators after evaluation: agents cannot mint tokens, and dispatch operations require valid tokens. Tokens are one mechanism for enforcement; others (signed commits, capability boundaries, formal verification traces) are equally valid.

\subsection{Session Boundaries and Resumability}
\label{sec:sessions}

AI-SDLC execution is not continuous: sessions pause (context limits, human availability, overnight breaks) and resume, and the protocol must maintain its guarantees across these boundaries.

\textbf{Session state:}
\begin{small}
\begin{align*}
\text{SessionState } \sigma_s &::= \langle \text{session\_id}, \text{mode}, \text{project}, \\
&\quad \text{checkpoint}, \text{parent\_session}? \rangle \\
\text{Checkpoint } c &::= \langle \text{current\_task}?, \text{decisions}, \\
&\quad \text{memory\_summary}, \text{timestamp} \rangle
\end{align*}
\end{small}

This prevents re-asking humans (decision churn) and losing memory context. Validator re-execution on resume is intentional: tokens are excluded from session state precisely because the context may have shifted during the pause, and revalidation against current state is required for soundness. Tokens are scoped to a single dispatch within an active session and do not persist across boundaries; implementations may additionally expire them within a session after a TTL to force revalidation when context changes substantially.

\subsection{Operational Semantics}
\label{sec:semantics}

\textbf{Validation:}
\begin{equation*}
\begin{array}{c}
V^* = \{v_1, \ldots, v_n\} \quad \text{phase} \in \{\text{pre\_execute}, \text{post\_execute}, \ldots\} \\
V_p = \{v \in V^* \mid v.\text{evaluation\_phase} = \text{phase}\} \\
\forall v \in V_p.\ v(\text{context}) \to r \\
\hline
\text{validate}(\text{context}, V_p, \text{phase}) \to \{\text{results}\}
\end{array}
\end{equation*}

\textbf{Token Issuance (Success):}
\begin{equation*}
\begin{array}{c}
\text{results} = \{r_1, \ldots, r_n\} \\
\text{apply\_policy}(D, \text{results}) = \text{proceed} \\
\hline
\text{issue\_token}(\text{results}, t) \to \\
\token = \langle \text{fresh\_id}, t, \text{sign}(\text{results}), \text{now}() \rangle
\end{array}
\end{equation*}

\textbf{Token Issuance (Refusal):}
\begin{equation*}
\begin{array}{c}
\text{results} = \{r_1, \ldots, r_n\} \\
\text{apply\_policy}(D, \text{results}) \neq \text{proceed} \\
\hline
\text{issue\_token}(\text{results}, t) \to \bot
\end{array}
\end{equation*}

In particular, a blocker forces $\text{apply\_policy} = \text{hard\_stop} \neq \text{proceed}$, so no token can be issued over blocker-carrying results (INV3).

\textbf{Dispatch with Token Enforcement:}
\begin{equation*}
\begin{array}{c}
\token = \langle \_, \text{task\_ref}, \_, \_ \rangle \\
\text{valid}(\token) \\
\text{current\_mode.tools contains dispatch} \\
\hline
\text{dispatch}(\text{task\_ref}, \token) \to \text{execute}(\text{task\_ref})
\end{array}
\end{equation*}

\begin{equation*}
\frac{\token = \bot}{\text{dispatch}(\text{task\_ref}, \token) \to \text{error}}
\end{equation*}

\textbf{Modes as hard boundaries.} There is no mode-transition rule, by design. Modes do not hand control to one another; they exchange only artifacts, and a change of mode is an explicit human action outside the engine:

\begin{axiom}[Mode boundary]
\label{ax:mode}
The engine executes exactly one mode per invocation: $\sigma.\text{current\_mode}$ is constant over every engine trace. A mode change is an external human action that begins a fresh trace (and is itself audited, per INV4). The only value crossing a mode boundary is an artifact, which is inert data: it carries no tokens and no capabilities. Tokens are scoped to a single dispatch within the issuing trace (Section~\ref{sec:sessions}) and cannot cross the boundary.
\end{axiom}

Axiom~\ref{ax:mode} formalises what the implementation enforces architecturally (Section~\ref{sec:implementation}): the output of the producer mode is input to the reviewer mode as data, never as authority, so WF1's tool disjointness yields Separation of Duties at the boundary.

\textbf{Blocked outcomes.} When token issuance yields $\bot$, execution proceeds in exactly one of two ways, determined by \emph{recoverability}: a repairable violation is delegated to a recursively spawned sub-instance (the Kleene step of Section~\ref{sec:kleene}), bounded by a recovery budget $R$; an absolute violation, or one whose budget is exhausted, escalates to a human. There is no third path; in particular, there is no rule by which a blocked task dispatches anyway. Let $d(t)$ denote the recovery depth of task $t$ (initially $0$).

\textbf{Recovery (subtask spawn):}
\begin{equation*}
\begin{array}{c}
\text{issue\_token}(\text{results}, t) = \bot \\
\forall r \in \text{results}.\ r.\text{severity} \neq \text{blocker} \qquad d(t) < R \\
L' = \text{spawn}(\text{subtask}(\text{results}, t),\ \langle M^*, V^*, D, C^* \rangle,\ \sigma.\text{current\_mode}) \\
\hline
\text{blocked}(t) \to L' \mathbin{;} \text{revalidate}(t) \quad [d(t){+}1]
\end{array}
\end{equation*}

The spawned instance $L'$ operates under the \emph{same} protocol tuple and the same mode as its parent; the parent may proceed only after $L'$ completes and $t$ is revalidated, so recovery re-enters the semantics at the validation rule rather than bypassing it.

\textbf{Escalation (human decision):}
\begin{equation*}
\begin{array}{c}
\text{issue\_token}(\text{results}, t) = \bot \\
\big(\exists r \in \text{results}.\ r.\text{severity} = \text{blocker}\big) \ \lor\ d(t) = R \\
\hline
\text{blocked}(t) \to \text{halt}\langle \text{results}, \text{policy\_decision} \rangle
\end{array}
\end{equation*}

The two premises are the two origins of a non-recoverable state: a validator whose criteria are absolute emits a blocker directly, and a repairable violation that exhausts its budget hardens into one --- at $d(t) = R$ the warn-level results are treated exactly as a blocker would be. Recoverability, not the disagreement policy, chooses between the rules; the policy determines only whether issuance refuses in the first place.

Escalation halts the trace with a structured payload (blocking validators, justifications, policy decision); resumption is a human action beginning a new trace, exactly as in Axiom~\ref{ax:mode}. Every rule above appends an entry to the audit log as a side effect (INV4); we leave the audit component implicit in the rule notation for readability.

The contrast between the two regimes is direct. Without enforcement, the agent interprets the policy and may comply or drift; execution proceeds either way, compounding process non-determinism with agent non-determinism. With enforcement, the mechanism gates execution: only compliant paths are possible, eliminating process non-determinism while bounding the residual non-determinism inherent to agent-based components.

\subsection{Bounded Non-Determinism}
\label{sec:bounded-nondeterminism}

The language does not produce deterministic \emph{outcomes}: agent-based components (validators, coders, the orchestrator) are inherently non-deterministic. What the mechanism eliminates is \emph{process} non-determinism, the variability from agents skipping or reinterpreting protocol steps. The distinction is sharp. With structural enforcement, dispatch is impossible without prior validation, so the validation step occurs deterministically (process non-determinism eliminated); the validator's own evaluation, once invoked, remains model-dependent (agent non-determinism bounded, not eliminated). Disagreement policies then turn variable validator outputs into deterministic decisions (proceed, escalate, halt). As with a type system, the mechanism does not make outcomes correct; it removes a whole class of failures structurally, leaving residual variability bounded and explicit. The honest claim is \emph{quasi-deterministic process behaviour}: outcomes that would otherwise be uncharacterisable become characterisable, auditable, and improvable.

\subsection{Failure Rate Bounds}
\label{sec:failure-bounds}

We characterise the quasi-deterministic improvement through a failure-rate analysis. Parameter values are illustrative; the qualitative conclusions hold across the plausible range.

\textbf{Setup and assumptions.} Model an AI-SDLC pipeline as $N$ sequential stages, each performed by an agent that introduces a defect with probability $p_i \in [0,1]$, independently across stages (\textbf{FR1}). An execution terminates in one of three outcomes: \emph{correct} (no defect ships), \emph{silent failure} (a defect ships undetected), or \emph{stall} (work halts awaiting a human decision). We track two metrics, $P_{\text{silent}}$ and $P_{\text{stall}}$, because the two delivery regimes distribute failure probability between them differently; collapsing them into a single ``failure rate'' obscures exactly the property enforcement provides. Validators are homogeneous with miss rate $p_v$ (\textbf{FR2}), and correlated through a \emph{common shock}: each validation event is systematic with probability $\rho \in [0,1]$ --- a shared blind spot of the underlying foundation models, on which all validators miss together with probability $p_v$ --- and independent otherwise (\textbf{FR3}). A recovery round is another attempt by the same agent, so its failure rate is tied to, and correlated with, the agent's own (\textbf{FR4}).

\textbf{Behavioural compliance.} Under behavioural delivery there is no enforced detection point: any validation happens at the discretion of the same agent, inside the same context, that produced the defect, so detection failure is correlated with generation failure and we conservatively credit no independent detection. Every introduced defect ships: $P_{\text{stall}}^{\text{beh}} = 0$ and, under FR1,
$$P_{\text{silent}}^{\text{beh}} \geq 1 - \prod_{i=1}^{N}(1-p_i).$$
The floor assumes stages fail independently; when process steps are skipped, downstream agents operate on corrupted state. We model the cascade by noisy-OR composition on the accumulated upstream error, with $\alpha_i \in [0,1]$ the susceptibility to upstream corruption:
$$p_{i+1}' = p_{i+1} + \alpha_i p_i' (1 - p_{i+1}), \qquad P_{\text{silent}}^{\text{beh}} \geq 1 - \prod_{i=1}^{N}(1-p_i'),$$
bounded in $[0,1]$, reducing to the independent floor at $\alpha_i = 0$ and rising toward saturation as $\alpha_i \to 1$.

\textbf{Structural enforcement.} Corollary~\ref{cor:noskip} supplies the premise the behavioural regime lacks: every stage is validated by the quorum, on every execution path. Under the common-shock assumption (FR3), the probability that a quorum of $K$ validators misses a defect is the mixture of the systematic and independent cases:
$$P_{\text{miss}}^{\text{quorum}} = \rho\, p_v + (1-\rho)\, p_v^K \;=\; p_v^K + \rho\,(p_v - p_v^K),$$
reducing to $p_v^K$ at $\rho = 0$ and to $p_v$ at $\rho = 1$. A defect at stage $i$ ships silently only if the quorum misses it. If caught, issuance refuses and repair is attempted through bounded recovery (Section~\ref{sec:semantics}), failing with probability $p_{\text{recovery}}$ --- a single rate that collapses the nested recovery loop, refined below. The per-stage failure mass therefore splits into a silent component and an unresolved component:
$$P_{\text{stage}_i} = \underbrace{p_i\, P_{\text{miss}}^{\text{quorum}(i)}}_{\text{silent}} \;+\; \underbrace{p_i\, (1 - P_{\text{miss}}^{\text{quorum}(i)})\, p_{\text{recovery}}}_{\text{unresolved}}, \qquad P_{\text{fail}}^{\text{str}} \leq \sum_{i=1}^{N} P_{\text{stage}_i}$$
by the union bound, giving $P_{\text{fail}}^{\text{beh}} / P_{\text{fail}}^{\text{str}} \geq [1 - \prod_i(1-p_i')] / \sum_i P_{\text{stage}_i}$. For realistic values ($p_i \approx 0.1$, $p_v \approx 0.2$, $K = 3$, $\rho \approx 0.3$, $p_{\text{recovery}} \approx 0.1$, $N = 5$), structural enforcement reduces failure rates roughly five-fold, widening past an order of magnitude at low $\rho$ and $p_{\text{recovery}}$ and narrowing as the behavioural rate saturates at large $N$. Section~\ref{sec:detection} validates both the quorum-miss formula and this ratio empirically.

\textbf{The silent-failure/stall trade.} The unresolved component does not ship: a violation still unrepaired at the recovery budget hardens into a blocker and the \textsc{Escalation} rule halts for a human (Section~\ref{sec:semantics}). This is the structural regime's defining trade:

\begin{proposition}[Silent-failure/stall trade]
\label{prop:trade}
Under structural enforcement with escalation, $P_{\text{silent}}^{\text{str}} \leq \sum_i p_i P_{\text{miss}}^{\text{quorum}(i)}$ and $P_{\text{stall}}^{\text{str}} \leq \sum_i p_i (1 - P_{\text{miss}}^{\text{quorum}(i)})\, p_{\text{recovery}}$; under behavioural compliance, $P_{\text{stall}}^{\text{beh}} = 0$ while $P_{\text{silent}}^{\text{beh}}$ grows cumulatively in $N$ toward saturation. Structural enforcement converts silent failures into visible, audited stalls; behavioural compliance converts stalls into silent failures.
\end{proposition}

\begin{corollary}[Benchmark invisibility of hard gates]
\label{cor:benchmark}
An autonomous, oracle-graded evaluation scores a stall as a failure, so it measures $P_{\text{silent}} + P_{\text{stall}}$ and is indifferent to the split between them. The governance value of a non-overridable gate --- driving $P_{\text{silent}}$ down at the cost of $P_{\text{stall}} > 0$ --- is therefore structurally invisible to such benchmarks: a hard gate can only lower a benchmark score, never raise it, whenever the stalled tasks include any the agent would have resolved.
\end{corollary}

Corollary~\ref{cor:benchmark} delimits what any autonomous benchmark, including our own evaluation (Section~\ref{sec:swebench}), can and cannot establish: it can measure whether an \emph{executed, validated process} improves outcomes, but the human-in-the-loop value of hard gates requires deployment data, not benchmarks (Section~\ref{sec:discussion}).

\textbf{Capability floor.} FR4 ties recovery to the agent: a recovery round is a re-attempt by the same agent with defect rate $p_a$, and successive attempts are correlated (same model, same blind spots), which we model with a second common shock $\rho_a$: over $R$ rounds,
$$p_{\text{recovery}}^{(R)} = \rho_a\, p_a + (1-\rho_a)\, p_a^R.$$
Substituting into the per-stage bound gives the \emph{advantage factor} of structural enforcement over the behavioural per-stage rate $p_a$:
$$A(p_a) \;=\; \frac{p_a}{P_{\text{stage}}} \;=\; \Big[\, P_{\text{miss}}^{\text{quorum}} + \big(1 - P_{\text{miss}}^{\text{quorum}}\big)\, p_{\text{recovery}}^{(R)} \,\Big]^{-1}.$$

\begin{corollary}[Capability floor]
\label{cor:floor}
$A(p_a)$ is decreasing in $p_a$ and $A(p_a) \to 1$ as $p_a \to 1$: as the agent's base failure rate approaches one, $p_{\text{recovery}}^{(R)} \to 1$ and the structural advantage vanishes. Enforcement multiplies the value of a capable agent; it cannot substitute for one.
\end{corollary}

The corollary predicts a floor on base capability below which enforcement yields no measurable benefit --- recovery rounds re-sample the same failure distribution, so a quorum that reliably catches defects merely converts them into unrecovered stalls. Section~\ref{sec:swebench} observes exactly this collapse on a weak coder ($\approx$15\% base resolve rate), with the enforcement loop demonstrably engaging yet producing no gain: the model's prediction, then the observation.

\textbf{Scope and validity.} The analysis is single-gate and conservative: the deployed protocol composes multiple independent gates (deterministic predicates, pre- and post-execution phases, per-mode gating, an orthogonal quorum), so the advantage compounds to the extent gates fail independently. The simplifying assumptions relax without changing the conclusion's direction: heterogeneous validators are bounded by substituting the worst per-validator rate $\max_j p_{v_j}$ for $p_v$ (the mixture derivation is unchanged); drifting per-stage rates are absorbed because the bounds are per-stage sums and products, so a time-varying $p_i(t)$ is bounded by its per-stage supremum; and an unstable correlation $\rho$ is handled by evaluating the bound at its upper envelope, since $P_{\text{miss}}^{\text{quorum}}$ is monotone in $\rho$. Two scope choices are explicit. The quorum-miss formula analyses the strictest gating ($\text{unanimous}$, $k = N$: a defect ships only if every validator misses); weaker thresholds $\text{quorum}(k)$, $k < N$, raise $P_{\text{miss}}^{\text{quorum}}$ monotonically, and the resulting policy trade-off is characterised empirically in Section~\ref{sec:policy-monte-carlo}. And $P_{\text{stall}}$ here counts defect-\emph{carrying} tasks only; stalls on defect-free work --- false blocks, including a misfiring absolute validator --- are a property of validator precision and policy strictness, measured separately in the same section, and the two stall costs add in deployment. Each substitution weakens the bound quantitatively but preserves it; the model bounds the average case rather than any specific deployment, and the bounds degrade gracefully rather than invert --- the ratio exceeds unity across the modelled range (Section~\ref{sec:detection}) --- narrowing as $\rho \to 1$, as $p_a$ approaches the floor of Corollary~\ref{cor:floor}, and as the behavioural rate saturates at large $N$. The structural difference is robust: enforcement bounds silent failure at a weighted product of agent and validator rates, while behavioural compliance permits cumulative, near-saturating growth.

\subsection{Enforcement Invariants}

The implementation must maintain INV1--INV5:

\textbf{INV1} (Token integrity):
\begin{equation*}
\begin{aligned}
\forall \token \in \text{issued\_tokens}.\ &\neg\text{can\_forge}(\token) \\
&\land \text{verify}(\token.\text{validator\_sigs})
\end{aligned}
\end{equation*}

\textbf{INV2} (Capability boundary):
\begin{equation*}
\begin{aligned}
\forall a \in \text{active\_roles},\ &t \in \text{Tools}. \\
\text{can\_call}(a, t) &\Rightarrow t \in \text{current\_mode.tools}
\end{aligned}
\end{equation*}

\textbf{INV3} (Non-overridable blockers):
\begin{equation*}
\begin{aligned}
\forall \text{results}.\ &(\exists r \in \text{results}.\\
&\quad r.\text{severity} = \text{blocker})\\
&\Rightarrow \text{issue\_token}(\text{results}, \_) = \bot
\end{aligned}
\end{equation*}

\textbf{INV4} (Audit completeness):
\begin{equation*}
\begin{aligned}
\forall \text{op} \in \{&\text{dispatch}, \text{validate},\\
&\text{transition}\}.\\
\text{execute}(\text{op}) &\Rightarrow \exists \text{log\_entry} \in \text{audit\_log}
\end{aligned}
\end{equation*}

\textbf{INV5} (Session resumability):
\begin{align*}
\forall c \in \text{checkpoints}. \text{resume}(c) &\Rightarrow \\
&\quad \text{preserve}(c.\text{decisions}) \\
&\quad \land \text{preserve}(c.\text{memory\_summary}) \\
&\quad \land \neg\text{rerun}(\text{completed\_tasks}(c))
\end{align*}

\subsection{The Enforcement Guarantee}
\label{sec:theorem}

The well-formedness conditions, operational semantics, and invariants combine into the language's central formal property: for a well-formed protocol, the invariants are not aspirations the implementation should uphold but consequences of the rule system --- no execution path exists on which they fail. The guarantee is conditional on three assumptions, each discharged by a specific implementation mechanism (Table~\ref{tab:claim-mapping}) and each stated explicitly because it marks a boundary of the formal claim.

\begin{assumption}[Unforgeable signing]
\label{as:crypto}
$\text{sign}(\cdot)$ is computable only by validators: agents cannot produce $\text{validator\_sigs}$ that $\text{verify}(\cdot)$ accepts. (Implementation: HMAC-signed single-use tokens; agents hold no signing key.)
\end{assumption}

\begin{assumption}[Boundary integrity]
\label{as:boundary}
Agents can invoke only tools exposed to their current mode: tool exposure is enforced at infrastructure level, below the agent. (Implementation: disjoint tool sets at the MCP boundary.)
\end{assumption}

\begin{assumption}[Engine mediation]
\label{as:mediation}
Every operation in $\{\text{validate}, \text{issue\_token}, \text{dispatch}, \text{spawn}, \text{halt}\}$ and every human mode action executes through the engine, which appends an audit entry for each. (Implementation: event-sourced, hash-chained audit log.)
\end{assumption}

\begin{theorem}[Enforcement soundness]
\label{thm:enforcement}
Let $\protocol = \langle M^*, V^*, D, C^* \rangle$ satisfy WF1--WF3. Under Assumptions~\ref{as:crypto}--\ref{as:mediation} and Axiom~\ref{ax:mode}, every engine trace of $\protocol$ satisfies INV1--INV4, and the guarantee is closed under Kleene composition: it holds for every execution in $L^*$.
\end{theorem}

\begin{proof}[Proof sketch]
By induction on trace length, with a case per rule of Section~\ref{sec:semantics}. The empty trace satisfies INV1--INV4 vacuously. For the inductive step, assume the invariants hold of the trace so far and consider the next rule applied. \textsc{Validate} issues no token and dispatches nothing; INV4 holds by Assumption~\ref{as:mediation}. \textsc{Issue-Token (Success)} creates a token carrying $\text{sign}(\text{results})$; by Assumption~\ref{as:crypto} no other rule (and no agent) can mint one, preserving INV1. Its premise $\text{apply\_policy}(D, \text{results}) = \text{proceed}$ is evaluable on all outcomes by WF4 as declared. \textsc{Issue-Token (Refusal)} returns $\bot$ whenever the policy outcome is not \text{proceed}; since a blocker forces \text{hard\_stop}, no issuance rule is applicable in the presence of a blocker --- INV3 is immediate, and no rule ordering can circumvent it. \textsc{Dispatch} fires only with a valid token and only when dispatch is among $\text{current\_mode.tools}$; the first premise means a validation preceded every execution (WF3 guarantees the validator set is nonempty), and the second, with Axiom~\ref{ax:mode} (mode constancy) and Assumption~\ref{as:boundary}, preserves INV2. \textsc{Dispatch} on $\bot$ yields an error and executes nothing. \textsc{Recovery} spawns $L'$ under the same tuple $\langle M^*, V^*, D, C^* \rangle$ and the same mode, so the induction hypothesis applies to $L'$ intact, and the parent resumes only through \textsc{Validate} --- this is the composition step that closes the guarantee under $L^*$ (Section~\ref{sec:kleene}); the budget premise $d(t) < R$ additionally bounds any single task's recovery chain, so a task cannot cycle between recovery and revalidation indefinitely --- it proceeds, or escalates by $d(t) = R$. \textsc{Escalation} halts the trace; nothing dispatches. Every case appends an audit entry by Assumption~\ref{as:mediation}, preserving INV4. INV5 is a cross-trace property of the checkpoint subsystem rather than of the rule system, and is treated separately (Section~\ref{sec:sessions}).
\end{proof}

\begin{corollary}[Non-skippability]
\label{cor:noskip}
In every trace of a well-formed protocol, every $\text{execute}(t)$ is preceded by a $\text{validate}$ step whose results issued the gating token. Validation cannot be skipped, reordered after execution, or overridden by any agent decision.
\end{corollary}

Corollary~\ref{cor:noskip} is the formal content of the ``process non-determinism eliminated'' claim of Section~\ref{sec:bounded-nondeterminism}, and the property the end-to-end evaluation (Section~\ref{sec:swebench}) probes empirically.

\textbf{Scope and honesty of the claim.} The theorem is deliberately conditional. It requires only WF1--WF3, which the verifier enforces unconditionally at load; WF4 and WF5 contribute no premise because their enforcement is partial (Section~\ref{sec:language}): WF4 admits sub-quorum policies with a warning, and WF5 checks declaration consistency, not logical satisfiability of constraint expressions. Consequently the theorem guarantees that validation \emph{occurs} and blockers \emph{halt}; it does not guarantee that the declared policy is sensible or the constraint set consistent --- a protocol author can still write a bad process, well-formed. Nor does it bound the \emph{quality} of validator judgment, which is the failure-rate model's subject (Section~\ref{sec:failure-bounds}). The assumptions likewise delimit the trust base: a compromised signing key voids INV1, and an agent with out-of-band tool access voids INV2 --- enforcement is structural exactly down to the infrastructure that implements it. Beyond the type-system analogy, the shape of the guarantee is that of noninterference results in language-based security~\cite{capability1974}: a syntactic discipline, checked mechanically, yielding a semantic property over all executions. Property-based testing corroborates the theorem's cases against the implementation: 17 property tests targeting INV1, INV4, WF1, severity-cap monotonicity, and the policy-halt behaviour, run at a budget of 10{,}000 randomised examples per test, with zero violations (Section~\ref{sec:implementation}). The core of the argument is additionally mechanised: a self-contained Lean~4 development (no external libraries, no unproven obligations) models the seven rules as an inductive step relation and proves INV3, INV2, and Corollary~\ref{cor:noskip} by trace induction, with Assumption~\ref{as:crypto} realised structurally --- a token over blocker-containing results is unconstructible in the model, not merely forbidden. The mechanisation covers the single-instance rule system (the Kleene composition step and INV4/INV5 are paper-only) and accompanies the artifact.

\subsection{The 2+N Pattern as Reference Configuration}
\label{sec:2pn}

We formalise the 2+N team pattern (two human-in-control, or HI-CTRL, roles plus N specialised agent members) as a reference protocol. Figure~\ref{fig:two_plus_n} shows the architecture with mode separation, role distribution, and tool boundaries.

\begin{figure}[htbp]
  \centering
  \includegraphics[width=\columnwidth]{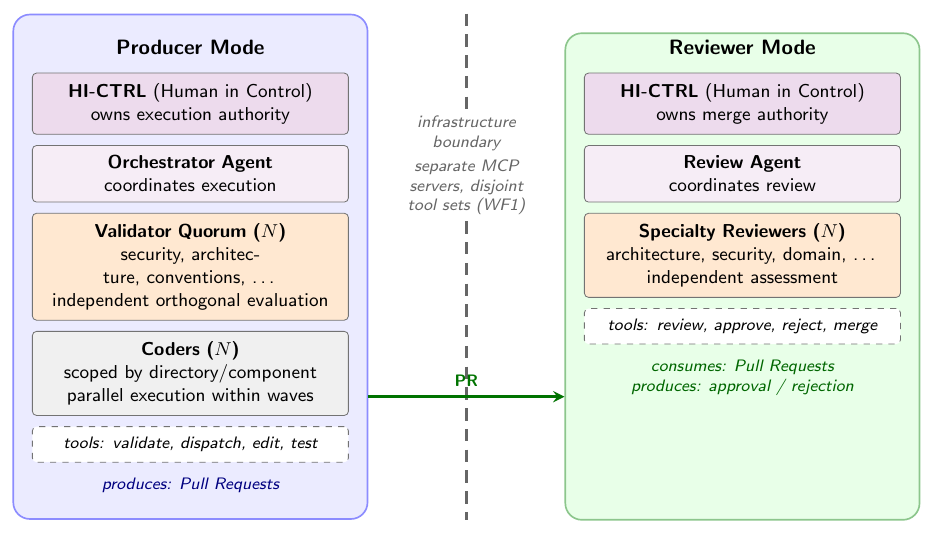}
  \caption{The 2+N team pattern: the ``2'' is the two human-in-control (HI-CTRL) roles, one per mode; the ``N'' is the specialised agents (validators, coders, reviewers). Producer and reviewer modes are separated by an infrastructure boundary --- separate MCP servers with disjoint tool sets --- so capability separation is enforced by WF1, not by convention; the only value that crosses is the PR artifact (Axiom~\ref{ax:mode}).}
  \label{fig:two_plus_n}
\end{figure}

\begin{lstlisting}[style=yaml, basicstyle=\footnotesize\ttfamily,
  aboveskip=3pt, belowskip=3pt,
  caption={The 2+N protocol as shipped (excerpt of the executable template
  \texttt{2+n.yml}; validator criteria abridged).},
  label={lst:2pn}, float=tbp, floatplacement=tbp]
protocol_id: "2+n"
modes:
  - mode_id: "producer"      # cannot merge or approve (WF1)
    tools: [edit, dispatch, resolve, test, validate]
    validators: [security, architecture, conventions,
                 quality, protocol_adherence, meta-spec]
    transitions: {complete: reviewer}
  - mode_id: "reviewer"      # cannot edit or dispatch (WF1)
    tools: [review, approve, merge, pr]
    validators: [security, architecture, quality, meta-spec]
    transitions: {approved: complete, rejected: producer}
validators:
  - validator_id: "security"
    evaluation_phase: "pre_execute"
    criteria:
      - "Check the task spec against OWASP Top 10 ..."
      - "Verify no hardcoded credentials, secrets, ..."
  - validator_id: "quality"
    evaluation_phase: "post_execute"
    # ... architecture, conventions, protocol_adherence,
    #     meta-spec elided
disagreement_policy: "quorum"
\end{lstlisting}

\textbf{Why this structure:}

\begin{itemize}
\item \textbf{Producer-reviewer separation} (WF1): producer mode cannot access merge tools, reviewer mode cannot access edit tools. This formalises the classical Separation of Duties principle~\cite{botha2001sod,nist80053} and prevents self-approval.
\item \textbf{Validator quorum}: validators with orthogonal criteria operate independently; unanimous agreement provides confidence, disagreement signals ambiguity requiring human judgment.
\item \textbf{Non-overridable blockers} (INV3): any blocker stops work regardless of other validator opinions.
\item \textbf{Human-in-control roles}: the two humans are first-class members with execution authority, ambiguity resolution, disagreement adjudication, and veto power. The ``2'' follows from the modes, not the other way around: a two-mode producer--reviewer process needs a human in control on each side of the boundary. The language admits the degenerate configurations explicitly rather than pretending: one person may hold the HI-CTRL role in both modes --- the boundary remains structurally enforced (tools stay disjoint; only the artifact crosses) but the separation of duties becomes nominal, a declared trade rather than a hidden one --- or a solo developer encodes a 1+N protocol with no review mode at all (producer through merge), accepting explicitly that the separation is gone (Section~\ref{sec:expressiveness}).
\item \textbf{Agent roles}: the N agents are first-class members with defined responsibilities within their capability boundaries.
\end{itemize}

\textbf{Scaling beyond 2+N.} The 2+N pattern is the smallest configuration in which Separation of Duties is structurally enforced rather than absent; the language places no upper bound on modes. Typical enterprise SDLC processes exhibit four to five distinct modes: planning, production, review, deployment, and operations, each enforcing structural separation of a different concern. Smaller teams may collapse modes (1+N for solo developers); larger organisations may decompose further. The language treats all configurations uniformly through the same mode/role/validator primitives, and we anticipate additional modes (specification authorship, compliance verification, post-deployment learning) as AI-SDLC practices mature.

\subsection{Expressiveness: Four Processes, One Language}
\label{sec:expressiveness}

A specification language earns its keep by what it can distinguish. We demonstrate expressiveness with four structurally different working processes, all shipped as executable templates with the reference implementation and all running unmodified on the same engine (Table~\ref{tab:expressiveness}). They are not toys: two are the production configurations of the implementation's own development, and two are the protocols of the end-to-end evaluation (Section~\ref{sec:swebench}).

\textbf{Solo (1+N).} A single producer mode holds the full tool set, edit through merge: Separation of Duties is \emph{deliberately} collapsed, and the governance burden shifts entirely to the validator quorum (four orthogonal validators plus a meta-level protocol-adherence validator). This is the language's lower bound: a solo developer keeps structural validation and audit without a second human, and the protocol says so explicitly rather than implicitly.

\textbf{Team (planner/producer/reviewer).} Three modes with disjoint tool sets: planning (\texttt{assess}, \texttt{plan}) gated by intent validators (clarity, scope, value-increment, agile conformance); production (\texttt{edit}, \texttt{dispatch}, \texttt{test}) unable to merge; review (\texttt{review}, \texttt{approve}, \texttt{merge}) unable to edit. This is the 2+N pattern of Section~\ref{sec:2pn} extended with an upstream planning mode --- the shape of a conventional enterprise SDLC, with WF1 enforcing at infrastructure level the separations that code review policies enforce by convention.

\textbf{Bug-fix surgeon (task-tailored minimality).} The protocol of the end-to-end evaluation. Its pre-execute validator is a deliberate \emph{pass-through} --- declared only to satisfy WF3, so the raw issue reaches the coder without inducing specification authoring --- and the enforcement lives post-execute: a minimality validator (severity cap \texttt{warn}) rejects sprawling or out-of-scope diffs and routes violations to recovery for a tighter fix, while a blocker-level global constraint (\texttt{no\_secrets\_in\_diff}) stays non-overridable. The protocol encodes a lesson the evaluation taught: a generic elaboration-inducing process \emph{reduced} bug-fix resolution, and the fix was to invert the enforcement target --- in configuration, not in code.

\textbf{Feature warden (containment over long horizons).} The surgeon's inversion, for the task type that fails the opposite way. Features fail by under-implementation and scope drift, so the pre-execute validator \emph{engages} structure (warn until the feature is decomposed into components with acceptance behaviour), and the post-execute containment validator checks completeness, rejects out-of-scope edits, and treats regression of previously working behaviour as its central rejection criterion (routed through recovery as shipped; a deterministic blocker-level regression gate over the previously-passing test set is designed for the feature-development study of Section~\ref{sec:discussion}). Minimality and completeness are opposite enforcement targets; the language expresses both as validator criteria plus severity policy, and switching between them is a template swap.

\begin{table}[htbp]
\centering
\caption{Four processes in one language. Each row is an executable template; ``target'' is what the post-execute enforcement optimises.}
\label{tab:expressiveness}
\small
\begin{tabular}{@{}lllll@{}}
\toprule
Protocol & Modes & SoD & Pre-execute & Target \\
\midrule
Solo (1+N) & 1 & collapsed & quorum (4+meta) & quality \\
Team & 3 & enforced (WF1) & intent gates & separation \\
Surgeon & 1 & n/a & pass-through & minimality \\
Warden & 1 & n/a & decomposition & completeness \\
\bottomrule
\end{tabular}
\end{table}

The comparison with adjacent formalisms is now concrete rather than positional. \emph{Prompts} can state all four policies --- and Section~\ref{sec:swebench} measures what that is worth: the same policy as prose moved nothing. \emph{Workflow and business-process languages} (BPMN and its human-agentic extensions) can express the team pipeline's sequencing and human tasks, but have no vocabulary for the parts doing the work here: infrastructure-level capability boundaries per mode, validator quorums with a severity algebra and pluggable disagreement policies, unforgeable evidence gating dispatch, or a declared severity cap that turns the same validator from a hard gate into a recovery trigger. \emph{Framework-embedded SOPs} (MetaGPT-style) encode roles in agent code, so the surgeon-to-warden inversion --- which here is a YAML swap under an unchanged engine, theorem intact --- would be a code change, and their gates bind by convention rather than by construction. The four templates are the policy/mechanism separation made concrete: one mechanism, four policies, and the guarantee of Theorem~\ref{thm:enforcement} holds for each because well-formedness, not intent, is what the verifier checks.

\section{Implementation}
\label{sec:implementation}

We have implemented this specification in a working system, demonstrating that the formal language is executable and the invariants maintainable.

\textbf{Architecture}:
\begin{itemize}
\item Protocol engine parses specifications and enforces well-formedness conditions.
\item Two MCP servers (producer-side, reviewer-side) implement mode separation (WF1).
\item Token subsystem implements validation token issuance and verification (INV1).
\item Capability boundaries enforced through tool exposure at infrastructure level (INV2).
\item Event-sourced audit log maintains immutable operation records (INV4).
\item Session checkpoint subsystem maintains resumable state across context boundaries (INV5).
\end{itemize}

\subsection{Architectural Properties}

Several properties emerged during implementation that strengthen the specification. The engine executes exactly one mode per invocation, halting at mode boundaries rather than auto-transitioning; modes communicate only through artifacts and mode changes are explicit human actions, enforcing human-in-control governance --- this is the architectural realisation of Axiom~\ref{ax:mode}, on which Theorem~\ref{thm:enforcement} rests. On any halt, the engine emits a structured, machine-parseable payload (halt type, blocking validators, justifications, policy decision) alongside human-readable output, enabling an external orchestrator to interpret a refusal and revise-and-retry or escalate. Validators are resolved through an open registry: third parties implement a single \texttt{evaluate(context)} interface and register a custom validator type without modifying engine code (the reference implementation includes a worked example). This extensibility also admits \emph{meta-level} validators registered like any other: the shipped protocol-adherence validator judges whether submitted work fits the current mode's declared profile, and the same interface admits validators over the audit trail itself (checking that validators are called when required, mode boundaries are respected, and dispatches token-gated) --- protocol self-auditing is expressible as an ordinary validator definition rather than a special construct. Validator severities form a total order ($\text{pass} < \text{warn} < \text{blocker}$); a severity cap lets validators-under-evaluation participate without blocking, with every cap recorded in the audit log.

\begin{table}[htbp]
\centering
\caption{Structural claims and their enforcement mechanism.}
\label{tab:claim-mapping}
\small
\begin{tabular}{@{}ll@{}}
\toprule
\textbf{Invariant / condition} & \textbf{Enforcement mechanism} \\
\midrule
INV1 (token unforgeability) & JWT HMAC, single-use, TTL \\
INV2 (capability boundary) & disjoint tool sets at MCP boundary \\
INV4 (audit completeness) & hash-chained audit log \\
INV5 (session resumability) & file-based checkpoint \\
WF1 (tool disjointness) & compiler verifier \\
WF3 (validator coverage) & pre-execute validator requirement \\
\bottomrule
\end{tabular}
\end{table}

\textbf{Self-development validation}: we have used the system to extend its own codebase under the 2+N protocol, providing initial evidence the approach is viable for real software development. \textbf{Property-based validation}: we verified audit-chain integrity (INV4), token unforgeability and single-use consumption (INV1), WF1 tool disjointness, severity-cap monotonicity, and the policy halt invariant against randomised inputs --- 17 property tests at a budget of 10{,}000 examples each (small discrete input spaces exhaust earlier), completing with zero violations in 17 minutes on a commodity development machine. The implementation is validated by 2{,}284 tests across three layers: unit/integration, end-to-end CLI journeys via subprocess, and property-based randomised testing of the core invariants; measured line coverage over the full codebase --- including operational tooling well outside the enforcement core --- is 68\%. A controlled end-to-end comparison against behavioural-instruction and bare baselines on real tasks is presented in Section~\ref{sec:swebench}; broader failure-mode analysis and quantification of human intervention rates remain future work.

\section{Empirical Evaluation}
\label{sec:evaluation}

The following studies characterise the enforcement properties of the specification language using the reference implementation as a test bed.

\subsection{Policy Monte Carlo: Disagreement Policy Trade-offs}
\label{sec:policy-monte-carlo}

We characterised the four disagreement policies via Monte Carlo simulation (25{,}000 trials per point, validator accuracy swept $p \in [0.5, 1.0]$, $N = 3$ validators), driving the actual \texttt{PolicyEvaluator} with simulated validator outputs drawn from real severity distributions.

\textbf{False-block / false-pass trade-off.} Figure~\ref{fig:policy_tradeoff} plots the false-block rate (good tasks refused) against the false-pass rate (bad tasks approved). At $p = 0.5$, \textsc{Unanimous} admits only 1.6\% of bad tasks while refusing 87.6\% of good ones; \textsc{Any} inverts this (10.5\% bad admitted, 52.5\% good refused); \textsc{Majority} is intermediate (6.1\%, 64.9\%). At $N = 3$, \textsc{Quorum} ($\lceil 0.67 \cdot 3 \rceil = 3$) coincides with \textsc{Unanimous}, consistent with the formal model; the two diverge for $N \geq 5$, illustrating that policy choice becomes a meaningful lever as validator count grows.
\textbf{Refusal mode decomposition.} Refusals of good tasks split into validator-level halts (blocker severity) and consensus-level escalations. At $p = 0.5$, \textsc{Unanimous}'s 87.6\% total refusal is 49.1\% HALT and 38.5\% ESCALATE, while \textsc{Any} reduces total refusal to 52.5\%, almost entirely HALT (49.7\%) with negligible escalation (2.8\%). Stricter policies escalate a larger share, raising friction on well-formed work.

\textbf{Quorum--Unanimous convergence.} The coincidence at $N = 3$ is a predicted consequence of the model ($\lceil 0.67 \cdot 3 \rceil = 3$ requires unanimity); the simulation reproduces it exactly, providing a checkable prediction against which the implementation can be validated.

\textbf{Implementation defect surfaced.} The simulation served as a test oracle: an absence of policy differentiation in initial runs flagged a defect (warnings counted as approvals, collapsing all four policies to identical behaviour). Correcting the evaluator restored the parametrised enforcement the model specifies, demonstrating that the analytical model provides a detectable contract against which implementations can be validated.

\begin{figure}[htbp]
  \centering
  \includegraphics[width=0.48\textwidth]{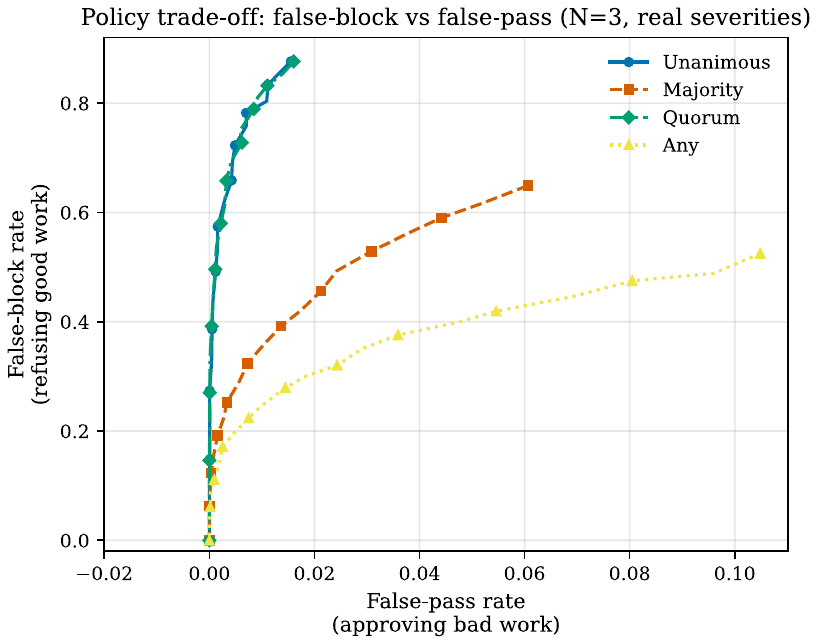}
  \caption{False-block versus false-pass rate for four disagreement policies ($N = 3$, real severity distributions). \textsc{Unanimous} and \textsc{Quorum} nearly eliminate false-passes at the cost of high false-block rates; \textsc{Any} accepts more bad work for lower friction on good work. Policy choice is a first-class enforcement parameter.}
  \label{fig:policy_tradeoff}
\end{figure}
\subsection{Governance Overhead}
\label{sec:overhead}

Structural enforcement adds runtime cost beyond inference. Microbenchmarks on the reference implementation put the CPU-bound operations (policy evaluation, protocol verification, token issuance and verification) all under 50\,µs, and the disk-bound durability operations backing INV4 and INV5 (audit append, checkpoint write) at 127\,µs and 672\,µs. End-to-end, a governed task path (governance, pre- and post-execute validation, token issuance and consumption, audit, checkpointing) adds 33.3\,ms over an ungoverned baseline, with every inference-class component --- coder, validators, and task classification --- stubbed, since inference is the cost against which overhead is compared, not a burden of enforcement. Against a representative 3\,s inference call this is 1.1\% overhead (teams can rescale as $33.3\,\text{ms}/T$ for their own inference latency $T$). The dominant costs are the durability operations, which deployments not needing audit persistence or checkpointing can omit.

\subsection{Detection Probability and the Structural Advantage}
\label{sec:detection}

This study grounds the failure-rate analysis of Section~\ref{sec:failure-bounds} empirically, validating the quorum-miss formula and comparing structural against behavioural failure as pipeline length grows.

\textbf{Quorum-miss formula validation.} Figure~\ref{fig:detection_quorum} compares the analytical $P_{\text{miss}}^{\text{quorum}} = p_v^K + \rho(p_v - p_v^K)$ against Monte Carlo simulation across $\rho \in [0, 1]$, for $K = 3$ and $K = 5$ at $p_v = 0.2$. Simulated points track the analytical curve throughout: at $\rho = 0$ the quorum benefits from independence ($P_{\text{miss}} = 0.008$ for $K = 3$), while at $\rho = 1$ it collapses to a single effective validator ($P_{\text{miss}} = 0.2$). The agreement confirms the implementation faithfully realises the model; whether the correlation assumption matches real cross-provider behaviour is a separate empirical question (Section~\ref{sec:limitations}).

\textbf{The independence ceiling.} The naive independent form $1 - (1 - p)^N$ is an \emph{upper bound} on detection that correlated validators cannot reach: any $\rho > 0$ shifts miss probability above the independent baseline. This is why $\rho$ appears in the model: it is not a conservatism knob but a structural consequence of validators sharing training data and architectural priors. Methodological diversity lowers $\rho$ but cannot drive it to zero.

\textbf{Structural versus behavioural failure by pipeline length.} Figure~\ref{fig:detection_structural} compares the structurally enforced rate (quorum plus recovery) against the behavioural cascading rate as $N$ grows. The gap widens with $N$: behavioural failure rises steeply (0.10 at $N = 1$, 0.45 at $N = 5$, 0.91 at $N = 20$) as corrupted state compounds, while structural failure grows close to linearly (0.016, 0.080, 0.318). At $N = 5$ this is a $5.6\times$ reduction, still nearly $3\times$ at $N = 20$ where the behavioural rate saturates. This is the empirical counterpart of the ratio claim in Section~\ref{sec:failure-bounds}.

\textbf{Ratio sensitivity.} Figure~\ref{fig:detection_sensitivity} maps the $P_{\text{behavioural}} / P_{\text{structural}}$ ratio across the $(\rho, p_{\text{recovery}})$ plane. The advantage is largest when validators are near-independent and recovery reliable (exceeding $100\times$ in the corner) and degrades gracefully otherwise. The ratio exceeds unity across the entire plane: structural enforcement never underperforms behavioural compliance within the modelled range, though the magnitude is an empirical property of the $\rho$ and $p_{\text{recovery}}$ a deployment exhibits.

\begin{figure}[htbp]
  \centering
  \includegraphics[width=0.48\textwidth]{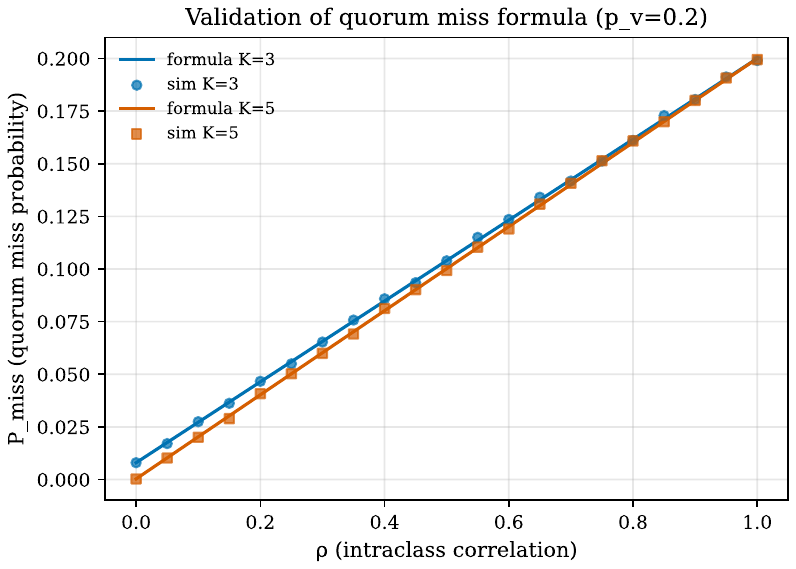}
  \caption{Validation of the quorum-miss formula ($p_v = 0.2$). Simulated points (markers) track the analytical curves (lines) for $K = 3$ and $K = 5$ across the full correlation range. At $\rho = 1$ the quorum provides no benefit over a single validator.}
  \label{fig:detection_quorum}
\end{figure}

\begin{figure}[htbp]
  \centering
  \includegraphics[width=0.48\textwidth]{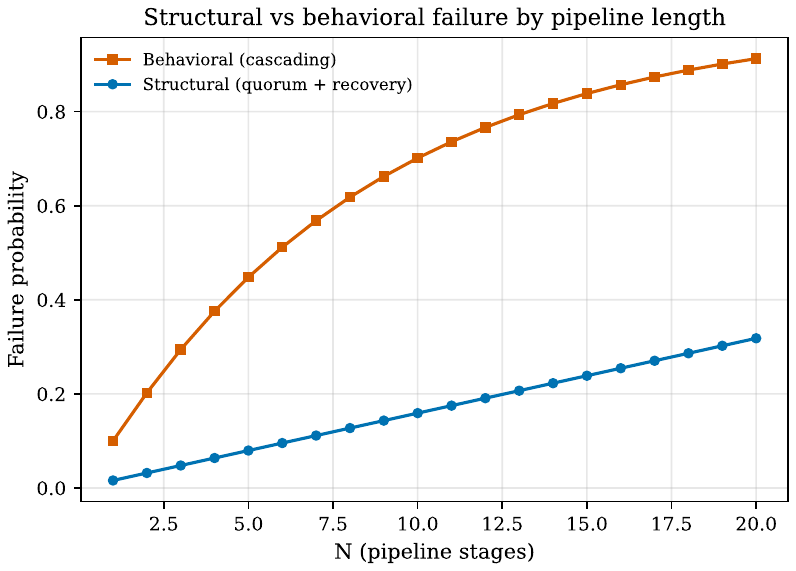}
  \caption{Structural versus behavioural failure probability by pipeline length $N$. Behavioural failure (cascading) rises steeply as corrupted state compounds; structural failure (quorum plus recovery) grows close to linearly.}
  \label{fig:detection_structural}
\end{figure}

\begin{figure}[htbp]
  \centering
  \includegraphics[width=0.48\textwidth]{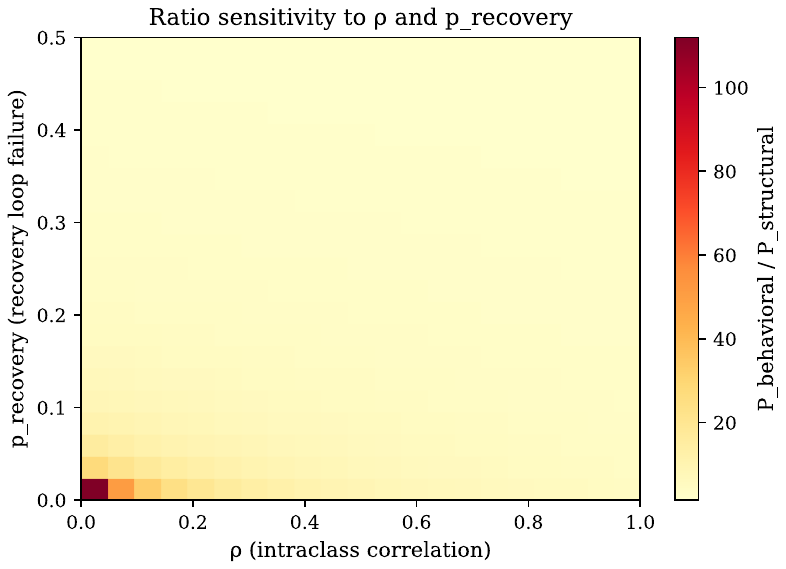}
  \caption{Sensitivity of the $P_{\text{behavioural}} / P_{\text{structural}}$ ratio to validator correlation $\rho$ and recovery-loop failure $p_{\text{recovery}}$. The ratio exceeds unity across the entire plane and is largest when validators are near-independent and recovery is reliable.}
  \label{fig:detection_sensitivity}
\end{figure}

\subsection{Byzantine Robustness}
\label{sec:byzantine}

The final study probes the limits of validator-quorum enforcement under adversarial validators that ignore their criteria and always emit a fixed verdict. We model two attack directions at $N = 5$, honest accuracy $p = 0.7$, sweeping compromised validators $f$ from 0 to 5.

\textbf{Always-pass (collusion / bypass).} When $f$ adversarial validators always approve, the false-pass rate rises with $f$ at a policy-dependent rate (Figure~\ref{fig:byzantine}, left). \textsc{Unanimous} is most robust: it still requires every honest validator to approve, so its false-pass rate stays below 1\% until $f = 3$ (2.2\%), remains 14.8\% even with four of five validators compromised, and reaches certainty only at $f = 5$. \textsc{Majority} and \textsc{Quorum} first exceed 1\% at $f = 2$ (2.2\% and 1.2\%); \textsc{Any} is the most fragile, leaking already at $f = 1$ (0.8\%). Degradation is graceful rather than cliff-edged --- no policy exceeds a one-in-three false-pass rate until a majority of the quorum is compromised --- and the ordering mirrors the false-block trade-off of Section~\ref{sec:policy-monte-carlo}: the policy most permissive of good work is most vulnerable to permissive adversaries.

\textbf{Always-block (denial of service).} When adversarial validators always emit a blocker, a single one ($f = 1$) drives the false-block rate to 100\% for all four policies (Figure~\ref{fig:byzantine}, right). This is INV3 operating exactly as specified: the unconditional rule that any blocker halts execution fires before policy aggregation, so no policy can out-vote it. What is a safety guarantee against missed defects is, under a compromised-restrictive validator, a denial-of-service vector.

\textbf{The INV3 asymmetry.} The always-block direction is \emph{policy-independent} (INV3 forces all policies to halt at $f = 1$); the always-pass direction is \emph{policy-dependent} (\textsc{Unanimous} tolerates up to two compromised validators). No single policy is robust to both directions simultaneously: the non-overridable-blocker semantics that make the system fail safe against missed defects necessarily make it fail closed under a restrictive adversary. This is the quantified cost of INV3 and bounds what validator-quorum enforcement can defend against: the mechanism assumes validators faithful to their criteria, not adversarial. Defending against compromised validators requires orthogonal mechanisms (validator attestation, anomaly detection on verdict distributions) outside the policy layer.

The fail-closed direction deserves a direct answer, because it looks like a denial-of-service vector: one hallucinating or over-conservative validator can halt all work. Three points bound the exposure. First, the failure is \emph{visible by construction}: in the terms of Proposition~\ref{prop:trade}, a false blocker produces an audited stall --- a halt with a structured payload naming the blocking validator and its justification, escalated to a human --- whereas the symmetric failure under behavioural compliance (a false pass) ships silently. The asymmetry is chosen: the recoverable failure mode is made cheap and loud, the unrecoverable one is made structurally hard. Second, the cost is bounded: a stall costs one human adjudication, and the audit log localises the offending validator immediately; a repeatedly-blocking validator is an operational signal, not a systemic corruption. Third, the language already carries the mitigations: severity caps allow a validator-under-evaluation to participate without blocking (every cap audited), disagreement policies other than \textsc{Any-Blocker} trade fail-closed strictness for availability, and meta-level validators can monitor verdict distributions for exactly this anomaly. What the language does not do is let an agent override the gate --- that is INV3, and softening it would surrender the property the rest of the paper establishes.

\begin{figure*}[t]
  \centering
  \includegraphics[width=0.95\textwidth]{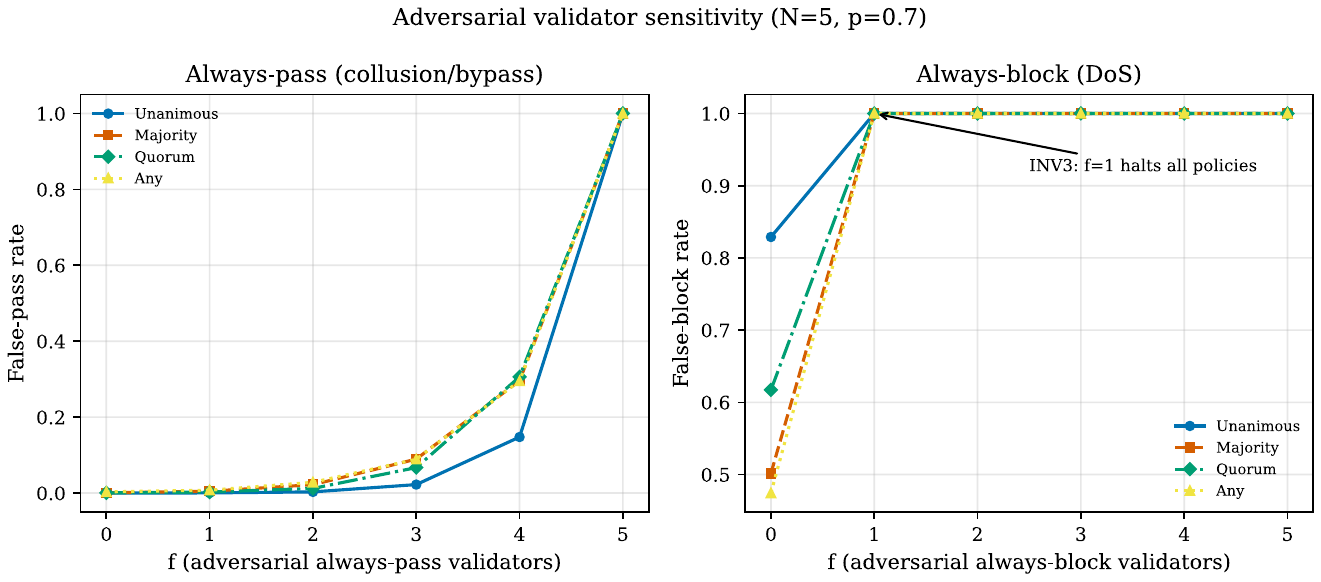}
  \caption{Adversarial validator sensitivity ($N = 5$, $p = 0.7$, 20{,}000 trials per cell). Left: under always-pass adversaries, false-pass rate is policy-dependent, with \textsc{Unanimous} most robust (below 1\% until $f = 3$) and \textsc{Any} most fragile (leaking from $f = 1$). Right: under always-block adversaries, a single compromised validator ($f = 1$) drives all policies to 100\% false-block: INV3 operating as specified.}
  \label{fig:byzantine}
\end{figure*}

%

\subsection{End-to-End Validation on Real Software Tasks}
\label{sec:swebench}

The simulation studies (Sections~\ref{sec:policy-monte-carlo}--\ref{sec:byzantine})
characterise the enforcement layer intrinsically. This study tests, end-to-end on
real software-engineering tasks scored by hidden acceptance suites, the empirical
premise that motivates the language: that a development methodology must be
\emph{executed} as a process, not merely \emph{told} to an agent. It is scoped by
Corollary~\ref{cor:benchmark}: an autonomous, oracle-graded benchmark measures
$P_{\text{silent}} + P_{\text{stall}}$ and therefore cannot reward the governance
value of a non-overridable gate, so the protocol arm runs in \emph{advisory}
configuration (violations route to recovery, never to a hard stop), and what the
study establishes is the value of the executed, validated process --- the
human-in-the-loop value of hard gates requires deployment data and is explicitly
future work (Section~\ref{sec:discussion}). The study deliberately exercises a
\emph{single} primitive --- a post-execute validator with token-gated dispatch and
a bounded recovery loop --- not the full apparatus (multi-validator quorums, the
2+N pattern, Byzantine tolerance), which is characterised in simulation above.

\subsubsection{Design}

We use SWE-bench Verified~\cite{swebench,swebenchverified}, a benchmark of real
GitHub issues paired with the maintainers' hidden regression tests; a task is
\emph{resolved} iff the generated patch makes the failing test pass without
breaking the passing suite. On each task we run arms that share an identical
coding agent and differ \emph{only} in how a fixed bug-fix methodology (``the
correct fix is the smallest change that makes the reported failure pass; do not
refactor, rename, or touch unrelated files'') is delivered:

\begin{itemize}
  \item \textbf{Arm A (bare)}: the agent receives the raw issue, no methodology.
  \item \textbf{Arm B (prose)}: the agent receives the methodology as prompt
        instructions --- the methodology \emph{told}, the standard ``encode
        process in the prompt'' approach (Section~\ref{sec:problem}).
  \item \textbf{Arm B$'$ (executed loop, ablation)}: the \emph{same} methodology
        run as an executed, validated, self-correcting process outside the
        protocol engine: a validator inspects the committed diff for minimality
        and localisation and routes violations back for a tighter fix (up to
        four attempts), implemented as plain advisory scaffolding with no
        tokens and no engine.
  \item \textbf{Arm C (executed protocol)}: the same methodology specified as a
        protocol and run by the engine in advisory configuration: the same
        post-execute validator and recovery loop as B$'$, with dispatch
        token-gated (INV1) and every step audited (INV4). No methodology text
        is added to the coder prompt beyond arm~B's, so B, B$'$, and C hold the
        policy \emph{content} constant and vary only the delivery mechanism.
\end{itemize}

Because the arms run on the \emph{same} tasks, the correct test is paired: we
report resolve rate per arm and the exact two-sided McNemar test on discordant
pairs. To isolate the effect of the coder, the enforcement machinery (the
validator and recovery classifier) is held fixed on a single model across all
runs, so differences between arms reflect the delivery mechanism, not a change
of enforcer. The protocol used is a task-tailored \emph{minimality} protocol: a
prior generic protocol that induced up-front specification authoring
\emph{reduced} resolve rate by inflating patch size (larger diffs resolve less),
which motivated inverting the enforcement target toward minimality rather than
elaboration.

\textbf{Setup and reproducibility.} All arms use the same agent scaffold the \texttt{opencode} CLI agent driving an open-weight coder; only the delivery of the methodology differs (A/B/C above). The enforcement machinery (the minimality validator and the recovery classifier) is fixed to \texttt{deepseek-v4-flash} across every run, so between-arm differences isolate the delivery mechanism, not the enforcer. Coders span two tiers of the DeepSeek-V4 family \texttt{deepseek-v4-flash} (cheap) and \texttt{deepseek-v4-pro} (strong) with \texttt{gpt-oss:20b} used for the capability-boundary condition. Task sets are frozen, seeded, repo-stratified draws from SWE-bench Verified (seeds 13, 29, and 47; per-run $n$ in Table~\ref{tab:swebench}), with one trial per cell (cells were near-deterministic, so budget went to more tasks rather than more trials). Patches are scored by the official SWE-bench evaluation harness under Docker, and the paired significance test is the exact two-sided McNemar on discordant pairs. The full patch text for every cell is retained, so each run is re-scorable from saved artifacts without re-dispatching the agent.

\subsubsection{Result: the executed process beats the told process}

Table~\ref{tab:swebench} reports four independent runs spanning two model tiers
(a cheap and a strong variant of one open-weight family) and task samples up to
$n{=}191$. The pattern is consistent, and it is two-sided. First, the null:
\textbf{methodology as prose does not improve outcomes} --- arm~B never
significantly beats bare in any run (all A-vs-B comparisons non-significant).
Second, the effect: \textbf{the same methodology executed as a validated,
self-correcting process (C) significantly outperforms both} bare and prose,
yielding a 14--22 point absolute gain in resolve rate. The identical policy
content, delivered two ways, produces a null and a large effect; the difference
is entirely in whether the process is executed or merely described.

\begin{table}[htbp]
\centering
\caption{Resolve rate (\%) on SWE-bench Verified across four paired runs: A bare, B prose instruction, C executed protocol (advisory). $p$-values are exact two-sided McNemar on discordant pairs. Rows are \texttt{deepseek-v4-flash}/\texttt{-pro} $\times$ task seed.}
\label{tab:swebench}
\small
\begin{tabular}{@{}lrrrrrr@{}}
\toprule
Run & $n$ & A & B & C & $p_{\text{A--C}}$ & $p_{\text{B--C}}$ \\
\midrule
flash, s13 & 94  & 59.6 & 55.3 & \textbf{69.1} & 0.064 & 0.004 \\
flash, s29 & 96  & 49.0 & 53.1 & \textbf{66.7} & 0.0002 & 0.015 \\
pro, s13   & 98  & 50.0 & 52.0 & \textbf{72.4} & \textless0.0001 & 0.0001 \\
flash, s47 & 191 & 50.3 & 50.3 & \textbf{67.0} & \textless0.0001 & \textless0.0001 \\
\bottomrule
\end{tabular}
\end{table}

The effect replicates across two independent task samples (seeds 13, 29), holds
at higher statistical power ($n{=}191$), and appears on both model tiers. The
loop genuinely engages rather than passing work through unchanged: across runs,
40--90\% of tasks required at least one recovery round, and arm~C uniquely
resolved clusters of tasks that both baselines missed. We are careful about what
this establishes. A validate-and-retry loop is a known effect class (iterated
generation with critique), and we do not claim its benefit as a novel result;
what the experiment isolates is that the \emph{same policy content} yields
nothing when told and a large, replicated gain when executed --- the empirical
motivation for a language whose subject is precisely the specification of
executable process. The guarantee the language adds over ad-hoc scaffolding is
the theorem's (Section~\ref{sec:theorem}): the loop \emph{cannot be skipped},
by construction, rather than happening to run.

\subsubsection{Ablation: the executed loop carries the gain}
\label{sec:bprime}

If the engine's contribution on this benchmark is the executed loop, then the
same loop implemented as plain advisory scaffolding --- no engine, no tokens ---
should reproduce the gain, and the engine in advisory configuration should add
nothing beyond it. That is exactly what the ablation shows. On the highest-power
sample (seed 47, $n{=}191$, re-scored as one pass across all four arms):
A 49.7\%, B 49.2\%, B$'$ 59.7\%, C 60.7\%.\footnote{These rates come from a
single four-arm scoring pass, so that every paired comparison in the ablation
shares one scoring environment; the three-arm seed-47 row of
Table~\ref{tab:swebench} was scored in a separate pass. SWE-bench scoring
applies patches and executes hidden test suites in per-task containers, and
absolute rates are sensitive to the scoring environment (arm~C: 67.0 there
vs.\ 60.7 here); the within-pass discordant-pair statistics are the quantities
of interest and are unaffected by this cross-pass variation.} Paired McNemar: B$'$ vs.\ B,
$p = 0.018$ --- the executed loop significantly beats the same methodology as
prose; B$'$ vs.\ C, $p = 0.90$ --- the engine in advisory configuration is
statistically indistinguishable from the bare loop. This is the expected
mechanistic identity, not a negative surprise: in advisory configuration the
engine \emph{is} the validated loop plus token-gating and audit, and
Corollary~\ref{cor:benchmark} says the additional properties it guarantees
(non-skippability under adversarial or drifting agents, auditability,
enforceable escalation) are precisely the ones an autonomous benchmark cannot
score. The ablation thus cleanly attributes the benchmark gain to the executed
process, and delimits the engine's further value to properties requiring
deployment evidence.

\subsubsection{Model commodity: the executed process compresses the model gap}

Running the strong (``pro'') and cheap (``flash'') tiers under both bare and
executed-process conditions on the same seed-13 task set yields a $2\times2$
grid (Table~\ref{tab:commodity}). Two effects stand out. First, \textbf{a cheap
model under the executed process outperforms a strong model bare}: on the 94
instances scoreable in both runs, flash-with-process resolves 69.1\% against
bare-pro's 48.9\% --- a paired per-instance comparison (discordant pairs 24/5,
exact McNemar $p = 0.0005$) --- so the process layer is worth more than a
model-tier upgrade. Second, the executed process \emph{compresses} the gap
between tiers: bare, the models differ (59.6 vs.\ 50.0); under the process,
they converge near the top (69.1 vs.\ 72.4). Which model to buy matters less than
what process it runs under (Section~\ref{sec:discussion}).

\begin{table}[htbp]
\centering
\caption{Model-commodity $2\times2$ (resolve rate, \%, SWE-bench Verified,
seed 13). The executed process lifts both tiers and compresses the gap between
them; the cheap model under the process beats the strong model bare.}
\label{tab:commodity}
\small
\begin{tabular}{@{}lcc@{}}
\toprule
 & Bare (A) & Executed (C) \\
\midrule
Cheap tier (flash) & 59.6 & 69.1 \\
Strong tier (pro)  & 50.0 & \textbf{72.4} \\
\bottomrule
\end{tabular}
\end{table}

\subsubsection{Boundary condition: the executed process requires a capable coder}

The advantage is not unconditional. Repeating the design with a small, weaker open-weight coder (\texttt{gpt-oss:20b}) collapses the effect: resolve rates fall to
A~9.1\%, B~19.2\%, C~14.1\% ($n{=}99$), the executed process no longer beats prose
($p_{\text{B--C}}{=}0.33$, $p_{\text{A--C}}{=}0.30$), and only B\,$>$\,A reaches
significance ($p{=}0.03$). This is also the one regime in which prose is not a
null, and the two findings are consistent rather than in tension: the prose
null of Table~\ref{tab:swebench} is a claim about coders \emph{above} the
capability floor, where the executed process dominates whatever the text adds;
below the floor, recovery rounds cannot convert critique into fixes, and
methodology text is the only guidance the coder can use at all. Diagnostically, the loop still
engaged --- 47 of 99 tasks exhausted all four recovery rounds and arm~C produced
the fewest empty patches --- but on a coder operating near 15\% base capability,
the process converts ``gave up'' into ``wrong answer'' more often than into
``correct fix.'' The executed process guarantees that validation occurs; it
cannot manufacture a correct patch from a coder incapable of producing one. This
is the collapse Corollary~\ref{cor:floor} predicts: recovery rounds re-sample
the same failure distribution, so as the agent's base failure rate grows the
advantage factor $A(p_a)$ falls to one regardless of how strictly the process
is enforced. The boundary condition was derived before it was observed, and the
diagnostics match the mechanism the model posits (the quorum caught defects;
recovery could not fix them).

\subsubsection{Threats to validity}

SWE-bench Verified predates the models' training cut-offs and is plausibly
contaminated; the paired design controls for this (contamination benefits all
three arms equally, so it cannot explain the between-arm gap), but absolute rates
should not be read as generalisation to unseen code. The evaluation covers a single task type (bug fixing) under one protocol and a single primitive; multi-validator quorums, the 2+N reviewer separation, and richer protocols are validated only in simulation, and feature development and other task types remain future work where the over-elaboration failure mode suggests the approach may have even more headroom. The benchmark's scope is bounded by Corollary~\ref{cor:benchmark}: it can establish the value of the executed process, and does; it cannot, by construction, measure the governance value of non-overridable gates or human-in-the-loop escalation, which require deployment usage data rather than oracle grading. Finally, the enforcer is held fixed on one model to isolate the coder effect; the interaction between coder and
enforcer capability is not yet mapped.

\subsection{Summary}
\label{sec:eval-summary}

Five studies characterise the enforcement layer along complementary axes. The
policy Monte Carlo (Section~\ref{sec:policy-monte-carlo}) shows disagreement
policy is a first-class trade-off between false-block and false-pass. The
overhead study (Section~\ref{sec:overhead}) shows the machinery costs roughly
1\% of inference latency. The detection study (Section~\ref{sec:detection})
validates the failure-rate model and confirms the structural advantage widens
with pipeline length. The Byzantine study (Section~\ref{sec:byzantine}) bounds
what the policy layer can defend against, exposing the INV3 asymmetry. Finally,
the end-to-end study (Section~\ref{sec:swebench}) tests the language's
empirical premise on real software tasks: on SWE-bench Verified, an identical
bug-fix methodology yields nothing when delivered as prose but a large,
replicated gain when executed as a validated, self-correcting process (two model
tiers, task samples up to $n{=}191$); an ablation attributes the gain to the
executed loop itself, exactly as Corollary~\ref{cor:benchmark} predicts for an
advisory configuration under oracle grading; the process compresses the
model-tier gap, with a cheap model under the process beating a strong model
bare; and a weak-coder boundary condition confirms the advantage is contingent
on adequate agent capability, as Corollary~\ref{cor:floor} predicts. Broader task types beyond bug fixing, and the
coder--enforcer capability interaction, remain the principal items of future
work.

\section{Discussion}
\label{sec:discussion}
\subsection{Limitations}
\label{sec:limitations}

Several limitations bound the claims.

\textbf{Scope of the end-to-end evaluation.} Beyond simulation and a self-development feasibility demonstration, Section~\ref{sec:swebench} provides a controlled end-to-end comparison against behavioural-instruction and bare baselines on SWE-bench Verified, scored by hidden acceptance suites. That evaluation covers a single task type (bug fixing) under one enforced protocol and a single open-weight model family; broader task types (notably feature development) and cross-family generalisation remain future work.

\textbf{Analysis under assumptions.} The failure rate bounds (Section~\ref{sec:failure-bounds}) assume failure-rate constancy, stable correlation, and bounded recovery cost; they establish the structural difference, but precise quantitative claims require empirical validation.

\textbf{Token integrity is implementation-dependent.} INV1 requires unforgeable tokens but not a specific mechanism: different contexts (single-process, distributed, multi-tenant) need different primitives. We treat this as a feature: the specification is portable, the enforcement contextual.

\textbf{Validator quality bounds system quality.} Enforcement guarantees that validation occurs, not that the validator is correct; the failure-rate advantage is contingent on validators clearing a useful detection threshold. The mitigation is compositional (deterministic predicates, pre- and post-execution phases, per-mode gating, an orthogonal quorum), but the problem is bounded, not solved: a sufficiently poor or correlated validator set degrades the guarantee however strictly the protocol is enforced.

\textbf{Applicability scales with project complexity, not team size.} The same primitives express single-agent and multi-mode configurations uniformly; what determines whether structural enforcement pays off is project complexity and stage. A simple project can be self-orchestrated by hand, but the manual cost of instructing, validating, and catching missed steps scales badly as complexity grows. We do not establish where this crossover lies.

\subsection{What the Language Does Not Solve}
Writing effective protocols still requires deep software engineering expertise; the language lets senior engineers encode that knowledge but does not supply it. Protocol evolution (who may change a protocol, what review a change requires) is governance the language does not prescribe, though its machine-executable form affords version control and peer review. Verifying temporal properties or stated goals, and synthesising optimal protocols from constraints, remain future work.

\section{Conclusion}

We have presented a specification language for AI-SDLC processes, addressing a documented gap: formal mechanisms for human-agent responsibility boundaries, approval gates, and governance constraints. The language is built on a single idea --- separating \emph{policy} (the declared modes of working) from \emph{mechanism} (structural enforcement via validation tokens, capability boundaries, and non-overridable gates) --- with formal abstract syntax, well-formedness conditions, operational semantics, and enforcement invariants. The separation buys two things at once. First, \emph{flexibility}: a process becomes a portable artifact, illustrated by the 2+N reference pattern with its producer--reviewer separation, orthogonal validator quorum, and non-overridable blockers, that teams reconfigure without touching agent code. Second, \emph{quasi-determinism}, now a theorem: for any well-formed protocol the enforcement invariants hold on every execution trace, closed under Kleene composition (Theorem~\ref{thm:enforcement}), so protocol steps cannot be skipped by construction. A failure-rate analysis makes the consequences precise --- silent failure bounded at a weighted product of agent and validator rates, the remainder converted into visible, audited stalls, and a capability floor below which the benefit vanishes --- validated in simulation and end-to-end. On SWE-bench Verified, an identical methodology yielded nothing delivered as prose instruction and a significant, replicated gain when executed as a validated, self-correcting process, with the capability floor observed where the model predicts it. That result is the thesis in miniature: because the protocol and its invariants are unchanged as models are swapped or upgraded, the advantage transfers across tiers, so as foundation models converge the durable engineering asset becomes the formally specified, executable process, not the model. The effect is not unconditional --- it requires a coder above the capability floor, and we have demonstrated it for bug fixing under one protocol --- so establishing it for feature development and across model families, and measuring in deployment the governance value of hard gates and human-in-the-loop escalation that oracle-graded benchmarks structurally cannot reward (Corollary~\ref{cor:benchmark}), is the natural next step.

\subsection*{Acknowledgments}
Generative AI tools (including Anthropic Claude, Google Gemini, DeepSeek, Qwen, and Kimi) were used, under the authors' direction, as assistive instruments in the preparation of this manuscript and its artifact --- text drafting and revision, LaTeX and figure preparation, coding assistance, and cross-checking the manuscript's claims against the artifact's outputs. All research content, design, and conclusions are the authors' own, and the authors take full responsibility for the work.

\subsection*{Data and Code Availability}

The reference implementation (protocol engine, verifier, MCP servers, and shipped templates), the Lean~4 mechanisation, the five simulation studies with their regenerated outputs, the property-test suite, and the SWE-bench experiment harness with saved per-task artifacts (every run re-scorable without re-dispatching agents) are released as open source (Apache 2.0) at \url{https://github.com/snodo-dev/snodo}, with the submission-time state archived at \url{https://doi.org/10.5281/zenodo.21967946}.

\bibliographystyle{ACM-Reference-Format}
\bibliography{references}

\end{document}